\documentclass[11pt]{article}
\pdfoutput=1
\usepackage{fullpage}

\usepackage{epsf}
\usepackage{fancyheadings,hyperref}
\usepackage{graphics}
\usepackage{graphicx}
\usepackage{psfrag}

\usepackage{amsmath, amsfonts, amsthm, amssymb, algorithm, graphicx}
\usepackage{mathtools}
\usepackage{comment}
\usepackage{enumitem}
\usepackage{pdfpages}
\usepackage[noend]{algpseudocode}
\usepackage{fancyhdr, lastpage}

\usepackage{listings}
 
\newcommand{\maxs}[1]{{\color{blue}{#1}}}

\usepackage{mathtools}


\algnewcommand\algorithmicinput{\textbf{INPUT:}}
\algnewcommand\INPUT{\item[\algorithmicinput]}
\algnewcommand\algorithmicoutput{\textbf{OUTPUT:}}
\algnewcommand\OUTPUT{\item[\algorithmicoutput]}

\newcommand{\calV}{\mathcal{V}}

\newcommand{\calM}{\mathcal{M}}
\newcommand{\calD}{\mathcal{D}}

\newcommand{\calL}{\mathcal{L}}

\newcommand{\calF}{\mathcal{F}}
\newcommand{\fraka}{\mathfrak{a}}

\newcommand{\calS}{\mathcal{S}}

\newcommand{\calA}{\mathcal{A}}
\newcommand{\calP}{\mathcal{P}}
\newcommand{\calX}{\mathcal{X}}
\newcommand{\calG}{\mathcal{G}}
\newcommand{\TV}{\mathrm{TV}}

\newcommand{\calN}{\mathcal{N}}

\newcommand{\Prit}{\mathbf{P}}
\newcommand{\Qit}{\mathbf{Q}}

\newcommand{\KL}{\mathsf{KL}}
\newcommand{\Sym}{\mathbb{S}}
\newcommand{\Alg}{\mathsf{Alg}}

\newcommand{\rmd}{\mathrm{d}}

\newcommand{\GOE}{\mathrm{GOE}}

\newcommand{\vj}{v^{(j)}}

\newcommand{\vtil}{\widetilde{v}}
\newcommand{\wtil}{\widetilde{w}}

\newcommand{\vi}{v^{(i)}}
\newcommand{\vk}{v^{(k)}}
\newcommand{\wk}{w^{(k)}}

\newcommand{\vkplus}{v^{(k+1)}}
\newcommand{\viplus}{v^{(i+1)}}
\newcommand{\vT}{v^{(T)}}

\newcommand{\wone}{w^{(1)}}

\newcommand{\viminus}{v^{(i-1)}}
\newcommand{\wi}{w^{(i)}}
\newcommand{\vone}{v^{(1)}}

\newcommand{\vktil}{\vtil^{(k)}}
\newcommand{\wktil}{\wtil^{(k)}}
\newcommand{\wjtil}{\wtil^{(j)}}
\newcommand{\wjplustil}{\wtil^{(j+1)}}
\newcommand{\vjplustil}{\vtil^{(j+1)}}

\newcommand{\wonetil}{\wtil^{(1)}}
\newcommand{\wiminustil}{\wtil^{(i-1)}}
\newcommand{\viminustil}{\vtil^{(i-1)}}
\newcommand{\witil}{\wtil^{(i)}}
\newcommand{\vonetil}{\vtil^{(1)}}

\newcommand{\R}{\mathbb{R}}
\newcommand{\I}{\mathbb{I}}
\newcommand{\Exp}{\mathbb{E}}
\newcommand{\Q}{\mathbb{Q}}

\newcommand{\Var}{\mathrm{Var}}

\renewcommand{\Pr}{\mathbb{P}}

{
 \theoremstyle{plain}
      
}
\theoremstyle{plain}
\newtheorem{thm}{Theorem}[section]

\newtheorem{lem}[thm]{Lemma}
\newtheorem{cor}[thm]{Corollary}

\newtheorem{prop}[thm]{Proposition}

\theoremstyle{definition}
\newtheorem{defn}{Definition}[section]

\newtheorem{exmp}{Example}[section]

\newtheorem{rem}{Remark}[section]

\newtheorem{obs}{Observation}[section]

\title{\Large \bf On the Gap Between Strict-Saddles and True Convexity: An $\Omega(\log d)$ Lower Bound for Eigenvector Approximation}

\author{ Max Simchowitz\thanks{UC Berkeley, CA. msimchow@berkeley.edu. } \and
Ahmed El Alaoui\thanks{UC Berkeley, CA. elalaoui@berkeley.edu. }
\and
Benjamin Recht\thanks{UC Berkeley, CA.  brecht@berkeley.edu.} }

\date{}
\begin{document}
\begin{titlepage}
\maketitle{}
\begin{abstract}

We prove a \emph{query complexity} lower bound on rank-one principal component analysis (PCA). We consider an oracle model where, given a symmetric matrix $M \in \R^{d \times d}$, an algorithm is allowed to make $T$ \emph{exact} queries of the form $w^{(i)} = Mv^{(i)}$ for $i \in \{1,\dots,T\}$, where $v^{(i)}$ is drawn from a distribution which depends arbitrarily on the past queries and measurements $\{v^{(j)},w^{(j)}\}_{1 \le j \le i-1}$. We show that for a small constant $\epsilon$, any adaptive, randomized algorithm which can find a unit vector $\widehat{v}$ for which $\widehat{v}^{\top}M\widehat{v} \ge (1-\epsilon)\|M\|$, with even small probability, must make $T = \Omega(\log d)$ queries. In addition to settling a widely-held folk conjecture, this bound demonstrates a fundamental gap between convex optimization and ``strict-saddle'' non-convex optimization~\cite{jin2017escape,ge2015escaping,lee2016gradient} of which PCA is a canonical example: in the former, first-order methods can have dimension-free iteration complexity, whereas in PCA, the iteration complexity of gradient-based methods must necessarily grow with the dimension. Our argument proceeds via a reduction to estimating the rank-one spike in a deformed Wigner model. We establish lower bounds for this model by developing a ``truncated'' analogue of the $\chi^2$ Bayes-risk lower bound of Chen et al.~\cite{chen2016bayes}. 


\end{abstract}
\end{titlepage}

\section{Introduction}
A major open problem in machine learning and optimization is identifying classes of non-convex problems that admit efficient optimization procedures. Motivated by the empirical successes of matrix factorization/completion~\cite{zhou2008large,rennie2005fast,lee1999learning}, sparse coding~\cite{olshausen1997sparse}, phase retrieval~\cite{fienup1982phase} and deep neural networks~\cite{krizhevsky2012imagenet,bahdanau2014neural}, a growing body of theoretical work has demonstrated that many gradient and local-search heuristics - inspired by convex optimization -  enjoy sound theoretical guarantees in a wide variety of non-convex problems~\cite{candes2015phase,boumal2016nonconvex,netrapalli2013phase,janzamin2015beating,arora2012computing,arora2015simple,sun2016geometric,sun2017complete,bhojanapalli2016dropping,zheng2015convergent,tu2015low}. Notably,  Ge et al.~\cite{ge2015escaping} introduced a polynomial-time noisy gradient algorithm for computing approximate local-minima of non-convex objectives which have the ``strict-saddle property'': that is, objectives whose first-order stationary points are either local minima, or saddle points at which the Hessian has a strictly negative eigenvalue. It has since been shown that many well-studied non-convex problems can be formulated as ``strict saddle'' objectives whose local minimizers are all globally optimal (or near-optimal)~\cite{sun2015nonconvex,lee2016gradient,ge2016matrix,bhojanapalli2016dropping}, thereby admitting efficient optimization by local search.

Recently, Jin et al.~\cite{jin2017escape} proposed a gradient algorithm which finds an approximate local minimum of a strict saddle objective in a number of iterations which matches first-order methods for comparable convex problems, up to poly-logarithmic factors in the dimension. This might seem to suggest that, from the perspective of first-order optimization, strict saddle objectives and truly convex problems are identical. But there is a caveat: unlike the algorithm proposed by Jin et al.~\cite{jin2017escape}, the iteration complexity of first order methods for optimizing truly convex functions typically \emph{has no explicit dependence on the ambient dimension}~\cite{bubeck2015convex,nesterov1998introductory}. This begs the question: 
\begin{quote} 
Does the iteration complexity of first order methods for strict saddle problems necessarily depend on the ambient dimension? Stated otherwise, is there a gap in the complexity of first-order optimization for ``almost-convex'' and ``truly convex'' problems? 
\end{quote}

This paper answers the above questions in the affirmative by considering perhaps the simplest and most benign strict saddle problem: approximating the top eigenvector of a symmetric matrix $M \in \R^{d\times d}$, also known as rank-one PCA. The latter is best cast as a strict-saddle problem with objective function to be maximized $v \mapsto v^{\top}Mv$, subject to the smooth equality constraint $\|v\|^2 = 1$~\cite{ge2015escaping,lee2016gradient}. We show that the gradient query complexity of rank-one PCA necessarily scales with the ambient dimension, even in the ``easy'' regime where the eigengap is bounded away from zero. 

More precisely, we consider an oracle model, where given a symmetric matrix $M \in \R^{d \times d}$, an algorithm is allowed to make $T$ \emph{exact} queries of the form $w^{(i)} = Mv^{(i)}$ for $i \in \{1,\dots,T\}$, where $v^{(i)}$ is drawn from a distribution which may depend arbitrarily on the past queries $\{v^{(j)},w^{(j)}\}_{1 \le j \le i-1}$; these queries are precisely the rescaled gradients of the objective $v \mapsto v^{\top}Mv$. We show (Theorem~\ref{MainTheorem}) that any adaptive, randomized algorithm which finds a unit vector $\widehat{v}$ for which $\widehat{v}^{\top}M\widehat{v} \ge \Omega(\|M\|)$ for any symmetric matrix $M$ whose second-eigenvalue is at most $\gamma$ times its leading eigenvalue in magnitude \emph{must} make at least $T = \Omega(\log d / \log (1/\gamma))$ queries. This matches the performance of the power method and Lanczos algorithms as long as $\gamma$ is bounded away from one. 

In fact, we show that if $T$ is bounded by a small constant times $\log d/\log (1/\gamma)$, then the probability of finding a unit vector $\widehat{v}$ with objective value at least $\Omega(\|M\|)$ is as small as $e^{-d^{\Omega(1)}}$.  We also show (Theorem~\ref{MainTheoremDetection}) that given any $\lambda \ge 2 + \Omega(1)$, it takes $\Omega(\log d / \log \lambda )$ adaptive queries to test if the operator norm of $M$ is above the threshold $\lambda$, or below $2 + o(1)$. Our lower bounds are based on the widely studied deformed Wigner random matrix model~\cite{lelarge2016fundamental,feral2007largest}, suggesting that the $\log d$ factor should be regarded as necessary for ``typical'' symmetric matrices $M$, not just for some exceptionally adversarial instances. 
\subsection{Proof Techniques}
We reduce the problem of top-eigenvector computation to adaptively estimating the rank-one component $\theta$ of the deformation $M =  \lambda \theta\theta^{\top} + W$ of a Wigner matrix $W$~\cite{anderson2010introduction,feral2007largest,lelarge2016fundamental} in our query model. Here, $\lambda$ controls the eigengap, and $\theta$ is drawn uniformly from the unit sphere. Unlike many lower bounds for active learning~\cite{garivier2016optimal,agarwal2009information,jamieson2012query}, it is insufficient to assume that the algorithm may take the most informative measurements in hindsight, since that would entail estimating $\theta$ with only $O(1)$-measurements. As a first pass, we use a recursive application of Fano's method, similar to the strategy adopted in Price and Woodruff~\cite{price2013lower} for proving lower bounds on adaptive estimation. This method bounds the rate at which information is accumulated by controlling the information obtained from the $i$-th measurement in terms of the information gained from measurements $1,\dots,i-1$. Unfortunately, in our setting, this technique can only establish a lower bound of $\Omega(\log d/\log \log d)$ queries.

To sharpen our results, we adopt an argument based on a $\chi^2$-divergence analogue of Fano's inequality, introduced in Chen et al.~\cite{chen2016bayes}. But whereas the $\KL$-divergence computations in Fano's inequality allow us to decompose the information obtained at each round into a sum, the adaptivity of the algorithm introduces correlations between the likelihood ratios that appear in the $\chi^2$ computations. Thus, we need to carefully truncate the distributions that arise in our lower bound construction, and restrict them to some carefully-defined ``good events''. This permit us to bound the rate of information-accumulation. 

In general, the probability of these good events conditioned on the spike $\theta$ may vary, and thus treating the truncated probabilities as conditional distributions introduces serious complications. To simplify things, we observe that the theory of $f$-divergences, from which Fano's inequality and the $\chi^2$-analogue in Chen et al.~\cite{chen2016bayes} are derived, can be generalized straightforwardly to non-normalized measures, i.e., truncated probability distributions. We therefore derive a general version of the Bayes-risk lower bound from Chen et al.~\cite{chen2016bayes} for non-normalized distributions, which, when specialized to $\chi^2$, enables us to prove a sharp lower bound of $\Omega(\log d)$ queries. 

To prove the lower bound on testing the spectral norm on $\|M\|$, we reduce the problem to that of testing the null hypothesis $M = W$ for a Wigner matrix $W$, against an alternative hypothesis  $ M = \lambda \theta \theta^{\top} + W $, where $\theta$ is drawn uniformly on the sphere and some sufficiently positive $\lambda$. The bound mainly follows from Pinsker's inequality (similarly to the combinatorial hypothesis testing lower bound in Addario-Berry et al.~\cite{addario2010combinatorial}) but again with the added nuance of the need to truncate our likelihood ratios due to the adaptivity of the algorithm. 
\subsection{Related Work}


\textbf{Oracle Lower Bounds for Optimizations.} In their seminal work, Nemirovskii and Yudin~\cite{nemirovskii1983problem} established lower bounds on the number of calls an algorithm must make to a gradient-oracle in order to approximately optimize a convex function. While they match known upper bounds in terms of dependence on relevant parameters (accuracy, condition number, Lipschitz constant), the constructions are regarded as brittle~\cite{Shamir_LB_No_Dim}: the construction considers a worst-case initialization, and makes the strong assumption that the point whose gradient is queried lies in the affine space spanned by the gradients queried up to that iterate. Arjevani and Shamir~\cite{arjevani2016oracle} addresses some of the weakness of the lower bounds~\cite{nemirovskii1983problem} (e.g., allowing some randomization), but at the expense of placing more restrictive assumptions of the class of optimization algorithms considered. In contrast, our lower bound places no assumptions on how the algorithm chooses to make its successive queries. 

In recent years, lower bounds have been established for stochastic convex optimization ~\cite{agarwal2009information,jamieson2012query} where each gradient- or function-value oracle query is corrupted with i.i.d.\ noise. While these lower bounds are information-theoretic, and thus unconditional, they do not hold in the setting considered in this work, where we are allowed to make exact, noiseless queries. As mentioned above, the proof strategy for proving lower bounds on the exact-oracle model is quite different than in the noisy-oracle setting.

\noindent\textbf{Active Learning and Adaptive Data Analysis.}  
Our proof techniques casts the eigenvector-computation as a type of sequential estimation problem, which have been studied at length in the context of sparse recovery and active adaptive compressed sensing~\cite{arias2013fundamental,price2013lower,castro2017adaptive,castro2014adaptive}. Due to the noiseless oracle model, our setting is most similar to~\cite{price2013lower}, whereas~\cite{arias2013fundamental,castro2017adaptive,castro2014adaptive} study measurement noise. Our setting also exhibits similarities to the stochastic linear bandit problem~\cite{soare2014best}. More broadly, query complexity has received much recent attention in the context of communication-complexity~\cite{anshu2017lifting,nelson2017optimal}, in which lower bounds on query complexity imply corresponding bounds against communication via lifting theorems. Similar ideas also arise in understanding the implication of memory-constraints on statistical learning~\cite{steinhardt2015memory,steinhardt2015minimax,shamir2014fundamental}. 

\textbf{Local Search for Non-Convex Optimization.}
As mentioned in the introduction, there has been a flurry of recent work establishing the efficacy and correctness of local search algorithms in numerous non-convex problems, including Dictionary Learning~\cite{arora2015simple,sun2017complete}, Matrix Factorization~\cite{bhojanapalli2016dropping,tu2015low,zheng2015convergent}, Matrix Completion~\cite{ge2016matrix,bhojanapalli2016global}, Phase Retrieval~\cite{boumal2016nonconvex,sun2016geometric,candes2015phase}, and training neural networks~\cite{janzamin2015beating}. Particular attention has been devoted to avoiding saddle points in nonconvex landscapes~\cite{lee2016gradient,ge2015escaping,jin2017escape,sun2015nonconvex}, which, without further regularity assumptions, are known to render the task of finding even local minimizers computationally hard~\cite{murty1987some}. Recent work has also considered second-order algorithms for non-convex optimization~\cite{sun2017complete,agarwal2016finding,carmon2016gradient}. However, to the best of the authors' knowledge, the lower bounds presented in this paper are the first which show a gap in the iteration complexity of first-order methods for convex and ``benign'' non-convex objectives.

\textbf{PCA, Low-Rank Matrix Approximation and Norm-Estimation.}
The growing interest in non-convexity has also spurred new results in eigenvector computation, motivated in part by the striking ressemblence between eigenvector approximation algorithms (e.g.\ the power method, Lanczos Algorithm~\cite{demmel1997applied}, Oja's algorithm~\cite{shamir2015fast}, and newer, variance-reduced stochastic gradient approaches~\cite{garber2016faster,shamir2015fast}) and analogous first-order convex optimization procedures. Recent works have also studied PCA in the streaming~\cite{shamir2015fast}, communication-bounded~\cite{garber2017communication,balcan2016communication}, and online learning settings~\cite{garber2015online}. More generally, eigenvector approximation is widely regarded as a fundamental algorithmic primitive in machine learning~\cite{jolliffe2002principal}, numerical linear algebra~\cite{demmel1997applied}, optimization, and numerous graph-related learning problems~\cite{spielman2007spectral,page1999pagerank,ng2001spectral}. While lower bounds have been established for PCA in the memory- and communication-limited settings~\cite{boutsidis2016optimal,shamir2014fundamental}, we are unaware of lower bounds that pertain to the noiseless query model studied in this work. 

Rank-one PCA may be regarded one of the simplest low-rank matrix approximation problems~\cite{horn2012matrix}. The numerical linear algebra community has studied low-rank matrix approximation far more broadly, with an eye towards computation-, memory-, and communication-efficient algorithms \cite{nelson2014deterministic,nelson2014time,razenshteyn2016weighted}, as well as algorithms which take advantage of the sparsity of their inputs~\cite{clarkson2017low,song2016low,clarkson2013low}. Previous work has also studied the problem of estimating functions of a matrix's singular values~\cite{li2016approximating}, including the special case of estimating Schatten $p$-norms~\cite{li2014sketching}. To the best of our knowledge, lower bounds for sketching concern the cases where the sketches are chosen non-adaptively.



\section{Statement of Main Results}
Let $\|\cdot\|$ denote the $2$-norm on $\R^{d}$, and let $\calS^{d-1} := \{x\in \R^{d}:\|x\| = 1\}$ denote the unit sphere. Let $\Sym^{d\times d}$ denote the set of symmetric $d\times d$ matrices, and for $M \in \Sym^{d\times d}$, we let $\lambda_1(M) \ge \lambda_2(M) \ge \dots \ge \lambda_d(M)$ denote its eigenvalues in decreasing order, $v_1(M),v_2(M), \dots, v_{d}(M)$ denote the corresponding eigenvectors, and overload $\|M\|$ to denote its operator norm.
\begin{defn}[Eigenratio]For $\gamma \in [0,1)$, we define the set of matrices with positive leading eigenvector and bounded \emph{eigenratio} between its first and second eigenvalues:
\begin{eqnarray}
\calM_{\gamma}:= \left\{M \in \Sym^{d\times d} :  \lambda_1(M) = \|M\| >0,  \frac{|\lambda_j(M)|}{\lambda_1(M)} \le \gamma \quad \forall j\ge 2\right\}.
\end{eqnarray}
\end{defn}
The iteration complexity of rank-one PCA is typically stated in terms of the eigengap $1 - \gamma$. This work instead focuses on lower bounds which hold when the eigengap is close to $1$, motivating our parameterization in terms of the eigenratio instead. We now define our query model:
\begin{defn}[Query Model]\label{Query_def} And \emph{adaptive query algorithm} $\Alg$ with \emph{query complexity} $T \in \mathbb{N}$ is an algorithm which, for rounds $i \in [T]$, queries an oracle with a vector $v^{(i)}$, and receives a noiseless response  $w^{(i)} = Mv^{(i)}$. At the end $T$ rounds, the algorithm returns a vector $\widehat{v} = v^{(T+1)} \in \calS^{d-1}$. The queries $v^{(i)}$ and output $\widehat{v}$ are allowed to be randomized and adaptive, in that $v^{(i)}$ is a function of $\{(v^{(1)},w^{(1)}),\dots,(v^{(i-1)},w^{(i-1)})\}$, as well as some initial random seed. We say that $\Alg$ is \emph{deterministic} if, for all $i \in [T+1]$, $v^{(i)}$ is a deterministic function of $\{(v^{(1)},w^{(1)}),\dots,(v^{(i-1)},w^{(i-1)})\}$. We say that $\Alg$ is \emph{non-adaptive} if, for all $i \in [T]$ the distribution of $v^{(i)}$ is independent of the observations $\{(w^{(1)}),\dots,w^{(i-1)})\}$, (but $\widehat{v}$ may dependent on past observations.) 
\end{defn}
\begin{exmp}
The Power Method and Lanczos algorithms~\cite{demmel1997applied} are both randomized, adaptive query methods. Even though the iterates $v^{(i)}$ of the Lanczos and power methods converge to the top eigenvector at different rates, they are nearly identical algorithms from our query-complexity perspective: both identify $M$ on the Krylov space $v^{(1)},Mv^{(1)},\dots,M^{T-1}v^{(1)}$. The only difference is that the Lanczos algorithm selects $\widehat{v}$ in a more intelligent manner than the power method. Running the power method from a deterministic initialization would be a \emph{non-randomized} algorithm. Any non-randomized algorithm, even an adaptive one, must take $d$ queries in the worse case, since $d-1$ queries can only identify a matrix up to a $d-1$ dimensional subspace. Randomized, but non-adaptive algorithms need to take $\Omega(d)$ queries as well, as established formally in Li et al.~\cite{li2014sketching}. 
\end{exmp}
\subsection{Lower Bound for Estimation}
\noindent We let $\Pr_{\Alg(M)}$ denote probability taken with respect to the randomness of $\Alg$ and a fixed $M$ as input, and $\Pr_{\Alg,M\sim \calD}$ denote probability with respect to $\Alg$ and $M$ drawn from a distribution $\calD$. The main result of this work is the following distributional lower bound:
\begin{thm}[Main Theorem] \label{MainTheorem} There exists universal positive constants $\gamma_0$, $\epsilon_0$, $d_0$, $c_1$ and $c_2$ such that the following holds: for all $\gamma \le \gamma_0$, $d \ge d_0$ and $\epsilon \le \epsilon_0$, there exists a distribution $\calD$ supported on $\calM_{\gamma}$ such that the output $\widehat{v}$ of any adaptive query algorithm $\Alg$ with query complexity $T$ satisfies 
\begin{eqnarray}
\Pr_{\Alg,M\sim \calD}\big[ \langle \widehat{v}, M\widehat{v} \rangle \ge (1-\epsilon)\|M\| \big]  \le c_1 \exp\left\{ -c_2d \cdot \left(\gamma/\gamma_0\right)^{2T}\right\}.
\end{eqnarray}
\end{thm}
\noindent Thus, since a distributional lower bound implies a worst-case lower bound, we have
\begin{cor} Any adaptive $T$-query algorithm with output $\widehat{v}$ which satisfies 
\[\Pr_{\Alg(M)}\left[ \langle \widehat{v}, M\widehat{v} \rangle \ge \Omega(\|M\|)\right] \ge \Omega(1),\] 
for all $M \in \calM_{\gamma}$ must make $T = \Omega( \log d/\log (1/\gamma))$ queries.

\end{cor}
The constraint that $\gamma \le \gamma_0$ implies that there is a large eigengap. In this regime, our lower bound matches the power method, which yields $\widehat{v}$ such that $\widehat{v}^\top M\widehat{v} \ge (1-\frac{1}{\mathrm{poly}(d)})\|M\|$ in $O(\log d/\log (1/\gamma))$ iterations.  We also note that, while $M$ is not necessarily positive semi-definite in our construction, one can simply add a multiple of $ |\lambda_2(M)|\cdot I$ to enforce this constraint\footnote{ Proposition~\ref{Reduction_Theorem} and the proof of Theorem~\ref{Main_Formal_Theorem} show that $\|\lambda_2(M)\|$ is bounded on $\calD$.}, and this only changes $\gamma$ by a constant.

\subsection{Lower Bound for Testing}
We now consider the problem of testing whether the operator norm of a symmetric matrix $M$ is below a threshold $\lambda_0$, or above a threshold $\lambda_1>\lambda_0$.
\begin{defn}[Adaptive Testing Algorithm] An adaptive detection algorithm makes $T$ adaptive, possibly randomized queries as per Definition~\ref{Query_def}, and at the end of $T$ rounds, returns a test $\psi \in \{0,1\}$ which is a function of $\{(v^{(1)},w^{(1)}),\dots,(v^{(T)},w^{(T)})\}$, and some initial random seed.
\end{defn}
Our second result establishes a lower bound on the sum of type-I and type-II errors incurred when testing between distributions on matrices with separated operator norms:

\begin{thm}[Detection Lower Bound] \label{MainTheoremDetection} There exists universal positive constants $\lambda_0$, $d_0$, $c_1$, $c_2$ such that the following holds: for all $d \ge d_0$ and $\lambda \ge \lambda_0$, there exists two distributions $\calD_0$ and $\calD_1$ on $\Sym^{d\times d}$ such that
\begin{eqnarray}
\Pr_{M \sim \calD_0}\left[\|M\| \le 2 + O\left(d^{-c_1}\right)\right] = 1 & \text{and} & \Pr_{M \sim \calD_1}\left[\|M\| \ge \lambda - O\left(d^{-c_1}\right)\right] = 1.
\end{eqnarray}
Moreover, for any binary test $\psi\in \{0,1\}$ returned by an adaptive $T$-query algorithm $\Alg$, we have
\begin{eqnarray}
\Pr_{\Alg,M \sim \calD_0}[\psi = 1] + \Pr_{\Alg,M \sim \calD_1}[\psi = 0] \ge 1 - \frac{\left(\lambda/\lambda_0\right)^{T}}{d^{c_2}}.
\end{eqnarray}
\end{thm}
\noindent This implies a worst-case lower bound for testing, matched by the power method for large $\lambda$:
\begin{cor} Any randomized adaptive $T$-query algorithm which can test whether $\|M\| \le 2 + o(1)$ or $\|M\| \ge \lambda = 2 + \Omega(1)$ with probability of error $1/10$ requires at least $T = \Omega(\log d / \log \lambda)$ queries. 
\end{cor}

\section{Reduction to Estimation\label{main_redux_sec}}
We now construct the distribution used to prove Theorems~\ref{MainTheorem} and~\ref{MainTheoremDetection}. We begin by constructing a family of distributions $\{\Pr_{\theta}\}$ on $\Sym^{d\times d}$, indexed by $\theta \in \calS^{d-1}$, and place a prior $\calP$ on $\theta$. We then show that if $M$ is drawn from the marginal distribution, then with good probability, $M $ lies in $\calM_{\gamma}$ for an appropriate $\gamma$, and that any $\widehat{v}$ for which $\widehat{v}^{\top}M\widehat{v}$ is large must be close to $\theta$. Hence, establishing the desired lower bound is reduced to a lower bound on estimating $\theta$. The construction is based on the Gaussian Orthogonal Ensemble, also know as the Wigner Model~\cite{anderson2010introduction}. 
\begin{defn}[Gaussian Orthogonal Ensemble (GOE)] We say that $W \sim \GOE(d)$ if the entries $\{W_{i,j}\}_{1 \le i \le j \le n}$ are independent, for $1 \le i < j \le n$, $W_{ij} \sim \calN(0,1)$, for $i \in [n]$, $W_{ii} \sim \calN(0,2)$, and for $1 \le j < i \le n$, $W_{i,j} = W_{j,i}$. We also define the constant 
\begin{eqnarray}
K_d := \frac{1}{\sqrt{d}}\Exp\left[\|W\|\right] = 2 + d^{-\Omega(1)}.
\end{eqnarray}
\end{defn}
For a precise, non-asymptotic upper bound on $K_d$, we direct the reader to Bandeira and van Handel~\cite{bandeira2016sharp}; an asymptotic bound can be found in Anderson et al.~\cite{anderson2010introduction}, and non-asymptotic bounds with looser constants are shown by Vershynin~\cite{vershynin2010introduction}. We now define the generative process for our lower bound:
\begin{defn}[Deformed Wigner Model]\label{Gener_Def} Let $\lambda > 0$, and $\calP$ a distribution supported on $\calS^{d-1}$ (e.g., the uniform distribution.) We then independently draw $\theta \sim \calP$ and $W \sim \GOE(d)$, and set $M = \lambda \theta \theta^{\top} + \frac{1}{\sqrt{d}}W$. We also let $\Pr_u$ denote the law of $M$ conditioned on $\{\theta = u\}$. 
\end{defn}
In the sequel, we will take our algorithm $\Alg$ to be fixed. Abusing notation slightly, we will therefore let $\Pr_u$ denote the law of $M$ and $\{(\vone,\wone),\dots,(v^{(T)},w^{(T)}),\widehat{v}\}$ under $\Alg$, conditioned on $\{\theta = u\}$. We now state our main technical result, which establishes a lower bound on estimating $\theta$ in the setting of Defintion~\ref{Gener_Def}, which we prove using Corollary~\ref{Main_Cor_estimation} in Appendix~\ref{Est_Cor_proof}.
\begin{prop}[Main Estimation Result]\label{estimation_cor_2}
Let $\lambda > K_d + o(1)$ and let $M$ be generated from the deformed Wigner model, Definition~\ref{Gener_Def}. Let $\widehat{v}$ be the output of an adaptive $T$-query algorithm with input $M$. Then for any $\eta \ge 0$, we have
\begin{eqnarray}
\Exp_{\theta \sim \calP}\Pr_{\theta}\left[ \langle \widehat{v}, \theta \rangle^{2} \ge \eta \right] \le \frac{2}{1-1/e}\cdot\exp\left\{\dfrac{-d \eta}{4 \left(c_1\lambda^2\right)^{T}} \right\},
\end{eqnarray}
where $c_1>8$ is a universal constant, (observe that $c_1\lambda^2>1$.)
\end{prop}
Proposition~\ref{estimation_cor_2} states that, until $\Omega(\log (\eta d)/\log \lambda)$ queries have been made, the probability of having an inner product with $\theta$ of at least $\eta$ is tiny, i.e., $O(e^{-(d\eta)^{\Omega(1)}})$. The following proposition establishes that, if $M$ is drawn from the deformed Wigner model, then with high probability, $M$ lies in $\calM_{\gamma}$ for $\gamma \approx 2/\lambda$, and that optimizing $\langle \widehat{v},M\widehat{v}\rangle$ entails estimating $\theta$:
\begin{prop}\label{Reduction_Theorem} Fix $\delta_0 \in (0,1)$, $\lambda > K_d + 2\sqrt{\log(1/\delta_0)/d} $ and $\theta \in \calS^{d-1}$. Let $M \sim \Pr_{\theta}$, then the following three assertions simultaneously hold with probability at least $1-2\delta_0$:
\begin{enumerate}
\item  $\lambda_{\max}(M) = \|M\| \ge \lambda - 2\sqrt{\log(1/\delta_0)/d}$, and for all $i \ge 2$, $|\lambda_i(M)| \le K_d + 2\sqrt{\log(1/\delta_0)/d}$,
\item $M \in \calM_{\gamma}$ for $\gamma = \gamma(d,\lambda,\delta):= \frac{K_d + 2\sqrt{\log(1/\delta_0)/d}}{\lambda - 2\sqrt{\log(1/\delta_0)/d}}$,
\item Let $\gamma = \gamma(d,\lambda,\delta)$ as above. For any $\epsilon \le 1-\gamma$ and any $w \in \calS^{d-1}$, if $ w^\top Mw  \ge (1-\epsilon)\theta^{\top}M\theta$, then
\begin{align}\label{Big_F_Def}
\left|\langle w, \theta \rangle\right| \ge F(\epsilon,\gamma) := \sqrt{\left(\frac{\gamma}{1-\gamma}\right)^2 + 1 - \frac{\epsilon}{1-\gamma}} - \frac{\gamma}{1-\gamma}.
\end{align}

\end{enumerate}
\end{prop}
\begin{rem}\label{Fgammaremark} As $\lambda \to \infty$, then $\gamma \to 0$ and $F(\epsilon,\gamma) \to \sqrt{1-\epsilon}$; thus for large $\lambda$, optimizing $\langle \widehat{v},M\widehat{v}\rangle$ is essentially equivalent to estimating $\theta$. The above proposition also lets us take $\lambda$ to be as small as $K_d + o(1)$, or equivalently, $\gamma$ arbitrarily close to 1. In this regime, we show in Appendix~\ref{f_fact_proof} that $F(\epsilon,\gamma)$ behaves like $\Theta(1-\gamma)$, provided that $\epsilon = O(1-\gamma)$. Thus, as the eigengap decreases, $\langle \widehat{v},M\widehat{v}\rangle$ must be ever-closer to $\|M\|$ to ensure that $\widehat{v}$ overlaps with the spike $\theta$. Nevertheless, we can still ensure non-negligible overlap between $\widehat{v}$ and $\theta$ for values of $\gamma$ \emph{arbitrarily close to 1}.
\end{rem} 

We now state a more detailed version of Theorem~\ref{MainTheorem}. A formal version of Theorem~\ref{MainTheoremDetection} is established in Section~\ref{sec:detection}.
\begin{thm}[Formal Statement of Theorem~\ref{MainTheorem}]\label{Main_Formal_Theorem}
There exist an absolute constant $c_1>0$, such that for any $\gamma \in (0,1/c_1)$ and $\epsilon \in (0,1-\gamma)$, there exists a distribution $\calD$ supported on $\calM_{\gamma}$ such that, for any randomized, adaptive $T$-query algorithm $\Alg$, we have
\begin{eqnarray}
\Pr_{M \sim \calD,\Alg}\big[\left\langle \widehat{v}, M \widehat{v}\right\rangle  \ge (1-\epsilon)\|M\| \big] \le 12\exp\left\{-\frac{d}{4} F(\gamma,\epsilon)^2 \cdot (c_1 \gamma)^{2T}\right\},
\end{eqnarray}
where $F(\gamma,\epsilon)$ is defined in Equation~\eqref{Big_F_Def}.
\end{thm}
If $\gamma$ is bounded away from $1$ and $\epsilon$ bounded away from zero, then by Remark~\ref{Fgammaremark}, the $F(\epsilon,\gamma) = \Omega(1)$, and we recover Theorem~\ref{MainTheorem} by observing that the quantity in the exponent is then $-\Omega(d) \cdot (\Omega(\gamma))^{2T}$. However, Theorem~\ref{Main_Formal_Theorem} is more general because, in view of Remark~\ref{Fgammaremark}, it permits $\gamma$ to be arbitrarily close to $1$. 
\begin{proof}[Proof of Theorem~\ref{Main_Formal_Theorem}] Set $\lambda =  \frac{K_d + 2d^{-1/2}}{\gamma} + 2d^{-1/2} \le \frac{K_d + 4d^{-1/2}}{\gamma}$. Now we apply Proposition~\ref{Reduction_Theorem} with this $\lambda$, and with $\delta_0 = 1/e$, so that the conditions of Proposition~\ref{Reduction_Theorem} hold with probability at least $1 - 2/e$. If $1 - \langle \widehat{v}, v_1(M) \rangle^2 \le \epsilon$, then Equation~\eqref{Big_F_Def} implies $\langle \widehat{v}, \theta \rangle^2 \ge  \eta := F(\gamma,\epsilon)^2$. Then, we invoke Proposition~\ref{estimation_cor_2} with that value of $\eta$, and condition on the event in Theorem~\ref{Reduction_Theorem}. Finally, we use the bound $\frac{2}{(1 - 1/e)(1-2/e)} \le 12$.
\end{proof}

\subsection{Computing the Conditional Likelihoods\label{Likelihood_Comp}}
We begin by introducing some useful simplifications. First, in the spirit of Yao's minimax duality principle~\cite{yao1977probabilistic}, we assume that $\Alg$ is deterministic\footnote{Indeed, let $p \in [0,1]$, and $E$ be any event measurable with respect to $\theta$, $M$, $\vone,\dots,\vT,\widehat{v}$. Then, given a randomized adaptive query algorithm $\Alg$ such that $\Pr_{\theta \sim \calP}\Exp_{\Alg,\Pr_{\theta}}[E] \ge p$, we can view $\Alg$ as a superposition of deterministic algorithms $\Alg_{\xi}$, where $\xi \in \Xi$ is a random seed. Then, $\sup_{\xi}\Exp_{\theta \sim \calP}\Pr_{\Alg_{\xi},M\sim\Pr_{\theta}}[E] \ge \Exp_{\xi \sim \calD_{\Xi}}\Exp_{\theta \sim \calP}\Pr_{\Alg_{\xi},M\sim\Pr_{\theta}}[E] = \Exp_{\theta \sim \calP}\Exp_{\Alg,M \sim \Pr_{\theta}}[E] \ge p$.}. Second, we assume that $v^{(1)},\dots,v^{(T)}$ are orthonormal. This is without loss of generality because one can reconstruct the response $w^{(i)}$ to a query $v^{(i)}$ by simply quering the projection of $v^{(i)}$ onto the orthogonal complement of the previous queries $v^{(1)},\dots,v^{(i-1)}$, and normalizing. 
Finally, we introduce a simplification which will make our queries resemble queries of the form $\wi = \theta \langle \vi, \theta \rangle + \text{ i.i.d.\ noise}$.
\begin{obs}\label{Observation_2}
For $i \in [T]$ will let $P_i$ then define the orthogonal project onto the complement of the span of $\{\vone,\dots,\vi\}$. We may assume without of generality that, rather that returning responses $w^{(i)} = Mv^{(i)}$, the oracle returns responses $w^{(i)} = P_{i-1} Mv^{(i)}$.
\end{obs}
This is valid because once $\Alg$ queries $v^{(1)},\dots,v^{(i-1)}$, it knows $M(I-P_{i-1})$, and thus, since $M$ and $P_{i-1}$ are symmetric, it also knows $(I-P_{i-1})M$. Thus, throughout, we will take $w^{(i)} = P_{i-1} Mv^{(i)}$. We let also let $Z_i := \{(\vi,\wi)\}_{1 \le j \le i}$ denote the data collected after the $i$-th measurement is taken, and let $\mathcal{F}_i$ denote the $\sigma$-algebra generated by $Z_i$. The collection $\{\calF_i\}_{1\le i \le T}$ forms a filtration, and since our algorithm is deterministic, $\viplus$ is $\calF_i$-measurable. We show that, with our modified measurements $w^{(i)} = P_{i-1}Mv^{(i)}$, then the query-observation pairs $(\vi,\wi)$ have Gaussian likelihoods conditional on $Z_i$ and $u$.
\begin{lem}[Conditional Likelihoods] \label{ConditionalLemma} Under $\Pr_{u}$, the law of $M$ conditioned on $\theta = u$, we have
\begin{multline}
P_{i-1}M\vi \big{|} Z_{i-1} \sim \mathcal{N}\left(\lambda (u^\top \vi) P_{i-1} u,\frac{1}{d}\Sigma_i \right), \quad \text{where} \quad \Sigma_{i} := P_{i-1}\left(I_d+ v^{(i)} v^{(i) \top} \right)P_{i-1}.
\end{multline}
In particular, $w^{(i)}$ is conditionally independent of $w^{(1)},\dots,w^{(i-1)}$ given $Z_{i-1}$ and $\theta = u$.
\end{lem}
Lemma~\ref{ConditionalLemma} is proved in Appendix~\ref{CondLemProof}. We remark that $\Sigma_i$ is rank-deficient, with its kernel being equal to the span of $\{v^{(1)},\dots,v^{(i-1)}\}$. Nevertheless, because the mean vector $\lambda (u^\top \vi) P_{i-1} u$ lies in the orthogonal complement of $\ker \Sigma_i$, computing $\Sigma_i^{-1} (\lambda (u^\top \vi) P_{i-1} u)$ can be understood as $\Sigma_i^{\dagger} (\lambda (u^\top \vi) P_{i-1} u)$, where $\dagger$ denotes the Moore-Penrose pseudo-inverse~\cite{horn2012matrix}. We write
\begin{eqnarray}
\|v\|_{\Sigma}^2 := v^\top\Sigma^{\dagger}v.
\end{eqnarray}
and we will use the following equality and inequality frequently and without comment: 
\begin{eqnarray}\label{InverseInequality}
\|\lambda v\|_{\frac{1}{d}\Sigma_i}^2 =d \lambda^2 \|v\|_{\Sigma_i}^2, & \text{and} & \|P_{i-1}u\|_{\Sigma_i}^2  \le \|u\|^2. 
\end{eqnarray}
These just follow from the facts that $P_{i-1}$ is an orthogonal projection and $\Sigma_i \succeq P_{i-1}$, and so $P_{i-1}\Sigma_i^{\dagger}P_{i-1} \preceq I$. 

\section{A First Attempt: a Lower Bound of $\Omega(\log d / \log \log d)$\label{first_attempt_sec}}
Many adaptive estimation lower bounds are often shown by considering the most informative measurements an algorithm could take if it knew the true hidden parameter~\cite{garivier2016optimal,jamieson2012query,agarwal2009information,simchowitz2017simulator}. 
Unfortunately, this line of attack in insufficient for a non-vacuous lower bound in our setting: if an oracle tells the algorithm to measure at a unit vector $v$ for which $\langle v, \theta \rangle $ is at least $\Omega(1)$, then we would have $\langle v, M v \rangle = \Omega(\|M\|)$, and so by Proposition~\ref{Reduction_Theorem}, we would verify that $v$ is close to $\theta$. Of course, it is highly unlikely that our \emph{first} measurement $v^{(1)}$ is close to the true $\theta$; indeed, if $\theta$ is drawn uniformly from the sphere $\calS^{d-1}$, then $\langle v, \theta \rangle^2  = O (1/d)$ with high probability. But what about the \emph{second} measurement, or the \emph{third}? What is to stop the algorithm from rapidly learning to take highly informative measurement? To show this cannot happen, we will adopt a simple recursive strategy:

\begin{enumerate}
\item We relate the information collected at stage $k$ to the inner products $\langle v_i, \theta \rangle^2 $, $i\in [k]$. 
\item We bound the inner product of $\theta$ with the $k+1$-st query by its inner products with all past queries as  
\[\langle v^{(k+1)}, \theta \rangle^2 \le C(\lambda) \cdot \sum_{i=1}^k \langle v^{(i)}, \theta \rangle^2,\] 
where $C(\lambda)$ is a constant depending on $\lambda$.
\end{enumerate}
To demonstrate the above proof strategy, we start by establishing a sub-optimal lower bound of $\Omega(\log d / \log \log d)$. Then in Section~\ref{sec_chi_est}, we introduce a more refined machinery to sharpen the bound to $\Omega(\log d)$.  First, we observe that the mutual information $I(Z_k;\theta)$ between $\theta$ and $Z_k$ (see, e.g., Cover and Thomas~\cite{cover2012elements}) is controlled by the inner products $\langle \vi, \theta \rangle^2 $, $i\in [k]$:
\begin{prop}\label{MutualInfoProp}
Let $\calP$ be an isotropic probability distribution supported on $\calS^{d-1}$, and $\Pr_{u}$ denote the law of $M = \lambda \theta\theta^\top + \frac{1}{\sqrt{d}}W$ conditioned on $\theta = u$. Then for all integers $k\ge 1$, 
\begin{eqnarray}
I(Z_k;\theta) \le \frac{\lambda^2}{2}\left( k + \sum_{i=1}^k\tau_i\right), & \text{where} & \tau_i :=\Exp_{u \sim \calP} ~ \Exp_{\Pr_{u}}\left[d\cdot\langle \vi, u \rangle^2\right].
\end{eqnarray}
\end{prop}
We prove Proposition~\ref{MutualInfoProp} in Appendix~\ref{Sec:WarmupAppend}. We now recursively bound the mutual information $I(Z_k;\theta)$, using an argument similar to Price and Woodruff~\cite{price2013lower}, with the exception that we will rely on a more recent continuum formulation of Fano's inequality~\cite{duchi2013distance,chen2016bayes} to control the information : 
\begin{prop}[Global Fano~\cite{duchi2013distance}]\label{Global_Fano}
Let $\calP$ be a prior over a measure space $(\Theta,\calG)$, and let $\{\Prit_{\theta}\}$ denote a family of distributions over a space $(\calX,\calF)$ indexed by $\theta \in \Theta$. Then, if $\calA$ is an action space, $\calL: \mathcal{A} \times \mathcal{X} \to \{0,1\}$ is a loss function, and $\fraka : \calX \mapsto \calA$ is a measurable map, we have
\begin{eqnarray}
\Exp_{\theta \sim \calP}\Pr_{X\sim \Prit_{\theta}}\left[\calL(\fraka(X),\theta) = 0\right] \le  \frac{I(X;\theta) + \log 2}{\log\left(1/\sup_{a \in \mathcal{A}}\Pr_{\theta \sim \calP}[\{L(a,\theta) = 0\}] \right)}.
\end{eqnarray}
\end{prop}
We will apply Proposition~\ref{Global_Fano} at each query stage $k$:  we let $\Theta = \calS^{d-1}$, $\theta$ denote the rank-one spike, $\calP$ the prior over $\theta$, $X$ to be the data $Z_k = \{(\vone,\wone),\dots,(\vk,\wk)\}$ collected at the end of round $k$, and $\Prit_{u}$ to be the law of $X = Z_k$ conditioned on $\theta = u$. We use the action space $\calA = \calS^{d-1}$, our actions will be the $k+1$-st query, $v^{(k+1)}$, and the loss function we consider is 
\begin{eqnarray}
\calL(v^{(k+1)},\theta) = \I\left(\langle v^{(k+1)},\theta \rangle^2 \ge \tau\right),
\end{eqnarray} for some fixed $\tau$. This leads to the following bound, proved in Appendix~\ref{Sec:WarmupAppend}.

\begin{prop}\label{KL_estimation_theorem} Let $\mathcal{P}$ denote an isotropic distribution on the sphere $\calS^{d-1}$, which satisfies the following concentration bound for some constants $C_1,C_2 >0$ and all $t \ge 0$,
\begin{eqnarray}\label{P_tail}
\sup_{v \in \calS^{d-1}}\Pr_{u \sim \mathcal{P}}\left[d\cdot\langle v,u\rangle^2 \ge t\right] \le e^{C_2-C_1t}.
\end{eqnarray}
Then, the sequence $\left\{\tau_i:= \Exp_{u_0 \sim \calP}\Exp_{\Pr_{u_0}}\left[d\cdot\langle \vi, u_0 \rangle^2\right], ~ i\ge 1\right\}$ satisfies the following recursion: for all $k\ge 1$, $t > C_2/C_1$, we have
\begin{eqnarray}\label{KL_Recursion}
\Exp_{u_0 \sim \calP}\Pr_{u_0}\left[d\cdot\left \langle \vkplus, u_0 \right\rangle^2 \ge t\right] \le \I(t\le d)\cdot\left(1 \wedge \frac{\log 2 + \frac{\lambda^2}{2} (k+\sum_{i=1}^k\tau_i)}{C_1t-C_2}\right).
\end{eqnarray}
Therefore, integrating over $t$ yields
\begin{eqnarray}\label{TauRec}
\tau_{k+1} &\le& \frac{C_2}{C_1} + \frac{\log 2 + \frac{\lambda^2}{2} (k+\sum_{i=1}^k\tau_i)}{C_1}\left(1 + \log \frac{C_1d}{\log 2 + \frac{\lambda^2}{2} (k+\sum_{i=1}^k\tau_i)}\right).
\end{eqnarray}
\end{prop}
In Appendix~\ref{Concentration}, we prove that if $\calP$ is the uniform distribution on $\calS^{d-1}$, then we can take $C_1= 1/8$, and $C_2 = 4$ in the above proposition. This relies on the following concentration result:
\begin{lem}[Spherical Concentration]\label{SphereConcentation} Let $\theta \sim \calP$ where $\calP$ is the uniform distribution over $\calS^{d-1}$. Then for all $v \in \calS^{d-1}$,
\begin{eqnarray}\label{SG_tail}
\Pr_{\theta \sim \calP}\left[\sqrt{d}\cdot \left|\langle v, \theta \rangle\right| \ge \sqrt{2} + t \right] \le e^{-t^2/2}.
\end{eqnarray}
\end{lem}
Hence, from Equation~\eqref{TauRec}, $\tau_k$ grows by at most $O(\log d)$ after each query. Hence, until $\log d/\log \log d$ queries are taken, $\tau_k$ will be $o(d)$, which entails that $\theta$ will not be accurately estimated. It is worth understanding why this spurious $\log \log d$ factor appears using $\KL$. The main weakness with Proposition~\ref{Global_Fano} is that the denominator contains the logarithm of the ``best-guess probability'' $\log\left(\sup_{a \in \mathcal{A}}\Pr_{\theta \sim \calP}[\{L(a,\theta) = 0\}]\right)$. This results in a very weak tail bound on $\Pr(\tau_k > t)$ of $O(1/t)$, which incurs a $\log d$ factor when integrated. To overcome this weakness, we will work instead with estimates based on the $\chi^2$ divergence, which will be a lot more careful in taking advantage of the small value of $\sup_{a \in \mathcal{A}}\Pr_{\theta \sim \calP}[\{L(a,\theta) = 0\}]$. 

\section{A Sharper Lower Bound on Estimation\label{sec_chi_est}}
In this section, we use more refined machinery based on the $\chi^2$-divergence to sharpen the lower bound from $\Omega(\log d/\log \log d)$ to $\Omega(\log d)$. When bounding the $\KL$-divergences in Proposition~\ref{MutualInfoProp}, the proof crucially relies upon the fact that the log-likelihoods decompose into a sum, and could thus be bounded using linearity of expectations. This is no longer the case when working with the squares of likelihood-ratios which arise in the $\chi^2$ divergence, because the adaptivity of the queries can introduce strong correlations between likelihood ratios arising from subsequent measurements. To remedy this, we will proceed by designing a sequence of ``good truncation events'' for each $\theta \in \calS^{d-1}$, on which the likelihood ratios will be well-behaved.  We now fix some positive numbers $\tau_1,\dots,\tau_{T+1}$ to be specified later, and define the events $A_u^{k}$ for $u \in \calS^{d-1}$ and integer $k$ by
\begin{eqnarray}\label{Truncation_Def_Eq}
A_u^{k} = \left\{\forall i \in [k]: d\cdot \langle u, \vi \rangle^2 \le \tau_i\right\} .
\end{eqnarray}
For an arbitrary probability measure $\Pr$ on a space $(\calX,\calF)$, and an event $A \in \calF$, we use the following notation to denote the truncated (non-normalized) measure
\begin{eqnarray}\label{Truncated_Distributions}
\Pr\left[B;A\right] := \Pr\left[B \cap A\right], &~~~ \forall B \in \calF.
\end{eqnarray}
In the sequel, we will be working with the measures $\Pr_u[;A_u^k]$. Note that these are no longer actual  probability measures, since their total mass is $\Pr_u[A_u^k]$, which is in general strictly less than one. In Appendix~\ref{Info_Tools}, we show that $f$-divergences - a family of measures of distance between distributions which include both the $\KL$ and the $\chi^2$ divergence \cite{csiszar1972class,chen2016bayes} - generalize straightforwardly to non-negative measures which are not normalized (e.g., truncated probability distributions.) Leaving the full generality to the appendix, we will use a ``generalized $f$-divergence'' between non-normalized measures which modifies the classical $\chi^2$-divergence (for a comparison to the classical $\chi^2$ divergence, see the discussion following Remark~\ref{measure_theory_remak}.)
\begin{defn}[$\chi^2+1$-divergence]\label{chi_plus_one_definition} Let $\Pr,\Q$ denote two nonnegative measures on a space $(\calX,\calF)$, such that $\Q[\calX] > 0$, and $\Pr$ is absolutely continuous with respect to $\Q$\footnote{That is, for every $A \in \calF$, $\Q[A] = 0$ implies that $\Pr[A] = 0$}. We define
\begin{eqnarray}
D_{\chi^2+1}(\Pr,\Q) := \int \left(\frac{\mathrm{d}\Pr}{\mathrm{d}\Q}\right)^2 \rmd\Q. \label{chi_p1_def_eq}
\end{eqnarray}
When $\Q$ is a probability distribution, the above can be written as $\Exp_{\Q}\left[\left(\frac{\mathrm{d}\Pr}{\mathrm{d}\Q}\right)^2\right]$.
\end{defn}

In Appendix~\ref{Info_Tools}, we prove a generalization of the $f$-divergence Bayes risk lower bounds of Chen et al.~\cite{chen2016bayes}, which we state here for the $\chi^2+1$ divergence: 
\begin{prop}\label{chi_sq_fano_prop}
Adopting the setup of Proposition~\ref{Global_Fano}, let $\calP$ be a distribution over a space $\Theta$, $\{\Prit_{\theta}\}_{\theta \in \Theta}$ be a family of probability measures on $(\calX,\calF)$, $\calA$ be an action space, $\calL:\calA \times S^{d-1} \to \{0,1\}$ a binary loss function, and let $\fraka$ denote a measurable map from $\calX$ to $\calA$. Given a family $\{A_{\theta}\}_{\theta \in \Theta}$ of $\calF$-measurable events, let $\Prit_{\theta}[\cdot;A_{\theta}]$ denote the truncated measure as per Equation~\eqref{Truncated_Distributions}. Set
\begin{eqnarray*}
V^{\fraka} := \Exp_{\theta \sim \calP}\Prit_{\theta}\left[\{\calL(\fraka(X),\theta) = 0\};A_{\theta}\right], & ~~\text{and}~~V_0 := \underset{a \in \calA}{\sup}~\Pr_{\theta \sim \calP}\left[\{\calL(a,\theta) = 0\}\right].
\end{eqnarray*}
Then, for any nonnegative measure $\Qit$ on $(\calX,\calF)$, we have
\begin{eqnarray}
	V^{\fraka} \le V_0 + \sqrt{V_0(1-V_0)\Exp_{\theta \sim\calP}D_{\chi^2 + 1}\left(\Prit_{\theta}\left[\cdot;A_{\theta}\right],\Qit\right)}. 
	\end{eqnarray}
\end{prop}
\begin{rem}\label{TruncationRemark} Even though the above proposition is an analogue of Corollary 7 in Chen et al.~\cite{chen2016bayes}, it cannot be derived merely as a consequence of that bound, and is sharper and easier to use than bounds that would arise by replacing the truncated distributions $\Prit_{\theta}[\cdot ;A_{\theta}^k]$ with conditional distributions $\Prit_{\theta}[ \cdot \big{|} A_{\theta}^k]$. See Remark~\ref{TruncationRemark2} for further discussion. 
\end{rem}


As in Section~\ref{first_attempt_sec}, we take $\Theta = \calS^{d-1}$, $\Prit_{u}$ to be the distribution of $Z_k$ conditioned on $\{\theta = u\}$, and as in our above discussion, $A_{\theta}^k$ will define the truncation events from Equation~\eqref{Truncation_Def_Eq}; the index $k$ for which we apply Proposition~\ref{chi_sq_fano_prop} will always be clear from context. If we were to follow the $\KL$ case, we would bound the corresponding mutual information quantity $\Exp_{\theta \sim \calP}D_{\chi^2 + 1}(\Prit_{\theta}[\cdot;A_{\theta}],\Qit)$ by taking $\Qit$ to be the law of $Z_k$ induced by $M$ when $M \sim \Exp_{\theta \sim \calP}\Pr_{\theta}$. We would then apply Jensen's inequality (in view of Lemma~\ref{fDivProperties}) to upper bound the corresponding mutual information quantity  by the average $\chi^2+1$-divergence $\Exp_{u_0,u_1 \sim \calP}D_{\chi^2+1}(\Prit_{u_0}[\cdot;A_{u_0}],\Prit_{u_1}[\cdot;A_{u_1}])$ between $\Prit_{u_0}$ and $\Prit_{u_1}$, where $u_0$ and $u_1$ are both drawn i.i.d.\ from $\calP$. 

This argument does not work in our setting, because once we restrict to the events $A_{\theta}^k$, two measures $\Prit_{u_0}[\cdot;A_{u_0}^k]$ and $\Prit_{u_1}[\cdot;A_{u_1}^k]$ may no longer be absolutely continuous, and thus have an infinite $\chi^2+1$-divergence. Instead, we apply Proposition~\ref{chi_sq_fano_prop} with the measure $\Qit:= \Prit_0$  to denote the (un-truncated) probability law of $Z_k$ under the random matrix $M =\frac{1}{\sqrt{d}}W$. Since $\Prit_0$ is un-truncated, and since the $\GOE$ matrix $W$ has a continuous density, all the measures $\Prit_{u}$, and thus $\Prit_u[\cdot;A_u^k]$, are absolutely continuous with respect to it. Thus, we can use the events $A_{u}^k$ to control the $\chi^2+1$ divergence as follows:
\begin{lem}[Upper Bound on Likelihood Ratios]\label{Estimation_LL_UB}  Let $u \in \calS^{d-1}$, and $A_u^k$ be the event in Equation~\eqref{Truncation_Def_Eq}. For $i \ge 1$ and $v^{(1)},\cdots,v^{(i)} \in \calS^{d-1}$, define the expected conditional likelihood ratio
\begin{eqnarray}
g_i\left(u;\{v^{(j)}\}_{1 \le j \le i}\right) &:=& \Exp_{\Prit_0}\left[\left(\frac{\mathrm{d}\Prit_u(Z_i |Z_{i-1})}{\mathrm{d}\Prit_0(Z_i | Z_{i-1}) }\right)^2 \Big{|} \big\{\vj\big\}_{1 \le j \le i}\right].
\end{eqnarray}
Moreover, let $\calV_{\theta}^k := \left\{(\vone,\dots,\vk) \in (\calS^{d-1})^{k}~:~ \forall i \in[k], ~d \cdot \langle \vi, \theta \rangle^2 \le \tau_i \right\}$. Then,
\begin{eqnarray}
D_{\chi^2+1}(\Prit_u,\Prit_0) = \Exp_{\Prit_0}\left[\left(\frac{\mathrm{d}\Prit_u(Z_k ; A_u^k)}{\mathrm{d}\Prit_0(Z_k)}\right)^2 \right] &\le& \sup_{\vone,\dots,\vk \in \calV_{u}^k} ~\prod_{i=1}^kg_i\left(u;\{v^{(j)}\}_{1 \le j \le i}\right).
\end{eqnarray}
\end{lem}
The above Lemma specializes Lemma~\ref{Generic_UB_LL}, proved in Appendix~\ref{Chi_sq_proof}. Noting that the conditional laws $\Prit_u(Z_i |Z_{i-1})$ and $\Prit_0(Z_i | Z_{i-1})$ have Gaussian densities, a computation detailed in Lemma~\ref{Chi_Squared_Computation_Lemma} yields
\begin{eqnarray}
g_i\left(u;\{v^{(j)}\}_{1 \le j \le i}\right)  &=&  \exp\left\{ \lambda^2\cdot d(u^\top\vi)^2 \cdot \|P_i u\|_{\Sigma_i}^2\right\}.
\end{eqnarray}
Using the bound $\|P_i u\|_{\Sigma_i}^2 \le \|u\|^2 = 1$ (Equation~\eqref{InverseInequality}), and that $d\left(u^\top\vi\right)^2 \le \tau_i$ on $A_u^k$, we get
\begin{eqnarray}\label{Chi_Plus_1_Eq}
\Exp_{\Prit_0}\left[\left(\frac{\mathrm{d}\Prit_u(Z_k ; A_u^k)}{\mathrm{d}\Prit_0(Z_k)}\right)^2 \right] &\le& \sup_{\vone,\dots,\vk \in \calV_{u}^k}~\prod_{i=1}^kg_i\left(u;\{v^{(j)}\}_{1 \le j \le i}\right) \le e^{\lambda^2 \sum_{i=1}^k \tau_i}.
\end{eqnarray}
Combining these bounds allows us to prove Theorem~\ref{Main_estimation_theorem} below, which establishes an analogue of the recursion given by Proposition~\ref{KL_estimation_theorem}.

\begin{thm}\label{Main_estimation_theorem} Let $\Pr_{\theta}$ denote the law of the sequential query model from the matrix $M = \lambda \theta \theta^{\top} + \frac{1}{\sqrt{d}}W$, where $W \sim \GOE(d)$, and $\theta \sim \calP$, where $\calP$ is a distribution supported on $\calS^{d-1}$. Let $\vone,\dots,\vT, \widehat{v}$ denote the queries and output of a deterministic algorithm $\Alg$. Then for any $\tau_1,\dots,\tau_{T+1} > 0$, 
\begin{multline}\label{Main_Est_eq}
\Exp_{\theta \sim \calP}\Pr_{\theta}\left[\exists k \in [T+1]: d \langle \vk , \theta \rangle^2 > \tau_k\right] \le \sup_{v \in \calS^{d-1}}\Pr_{\theta \sim \calP}\left[d\langle v, \theta \rangle^2 \ge \tau_1\right] \\
+ 2\sum_{k=2}^{T+1} e^{\frac{\lambda^2}{2} \sum_{i=1}^{k-1}\tau_i}\cdot\sqrt{\sup_{v \in \calS^{d-1}}\Pr_{\theta \sim \calP}\left[d\langle v, \theta \rangle^2 \ge \tau_k\right] }.
\end{multline}
\end{thm}
The above theorem gives an upper bound on the probability that the inner products $\langle \vi, \theta \rangle^2$ grow faster than the sequence $\{\tau_i\}$. The summands on the right-hand side of Equation~\eqref{Main_Est_eq} exhibit a tradeoff between the values $\tau_1,\dots,\tau_{k-1}$, which capture the information gathered up to round $k$, and, as in Section \ref{first_attempt_sec}, the term $\sup_{v \in \calS^{d-1}}\Pr_{\theta \sim \calP}[d\langle v, \theta \rangle^2 \ge \tau_k]$, which represent how likely the best guess of $\theta$ one could make without making any queries. We prove this theorem in Appendix~\ref{sec:Main_estimation_thm_proof}

When we specialize Theorem~\ref{Main_estimation_theorem} by taking $\calP$ to be the uniform measure on the sphere, we obtain the  following corollary, established in Appendix~\ref{Main_Cor_Proof}. 
\begin{cor}\label{Main_Cor_estimation} Let $\calP$ be the uniform measure on the sphere $\calS^{d-1}$. Then, for any $\delta \in (0,1)$,
\begin{multline}
\Exp_{\theta \sim \calP}\Pr_{\theta}\left[\exists k \in [T+1]:  \langle \vk, \theta \rangle^2 \ge  \left(2\lambda^2 \cdot c(\delta,\lambda)\right)^{k-1}\cdot \frac{2\left(\sqrt{\log(1/\delta)} + 1\right)^2}{d}\right] \le \frac{2\delta}{1-\delta},
\end{multline}
where $c(\delta,\lambda) = (1 + 1/\lambda^2)\left\{(1 - 1/2\lambda^2)(1 - \sqrt{1/(1+\log(1/\delta))})\right\}^{-1} = 1 + o(1/\lambda) + o(\delta)$.
\end{cor}
We observe that in Corollary~\ref{Main_Cor_estimation}, the inner products $\langle \vk,\theta \rangle^2$ are unlikely to grow faster than $\Omega\left(\frac{\lambda^{2(k-1)}}{d}\right)$ with high probability. Thus, in order for $ \langle \vk, \theta \rangle^2$ to be at least $\Omega(1)$ at one of the $T$ iterates, one needs at least $T = \Omega(\log d/ \log \lambda)$ iterations. This is the insight which underlies the proof of the main technical result, Proposition~\ref{estimation_cor_2}, which is deferred to Appendix~\ref{Est_Cor_proof}.

Examining the proof of Corollary~\ref{Main_Cor_estimation}, the base of $2\lambda^2$ in the exponent is essentially the best we can hope from our techniques. While this base leads to the order optimal $\log d / \log \lambda$ sample complexity in the large $\lambda$ regime, we see that as $\lambda$ approaches $K_d \approx 2$, and thus our upper bound on $\gamma$ from Proposition~\ref{Reduction_Theorem} approaches $1$, our lower bound still permits $\langle \vi,\theta \rangle^2$ to grow at a rate of, say, $8^k$. A classical result in random matrix theory~\cite{feral2007largest} states that the first eigenvector of $M$ correlates with the spike $\theta$ as soon as $\lambda>1$. In this regime, our lower bound allows $\langle \vi,\theta \rangle^2$ to grow at a rate of about $2^T$.  This appears to be loose, because it a) a does not rule out a fast rate of convergence despite a vanishing eigengap and b) does not reflect that the rank one perturbation $\theta\theta^{\top}$ no longer correlates with $\theta$ if $\lambda < 1$~\cite{feral2007largest}.

\section{Testing the Operator Norm\label{sec:detection}}
We retain the notation of the previous section, letting $A_u^{k}$ denote the truncation events from Equation~\eqref{Truncation_Def_Eq}, and $\Pr_u[\cdot;A_u^{k}]$ denote the corresponding truncated measures. We are interested in testing whether $M$ is drawn from $\Pr_{0}$, or from $\Pr_{\theta}$ for some $\theta$ and $\lambda>0$. Because the test is a measurable function of $Z_T$, we establish a lower bound on testing between the distribution $\Prit_{0}$ and $\overline{\Prit}:= \Exp_{\theta \sim \calP}\Prit_{\theta}$; where again, for $\theta \in \{0\} \cup \calS^{d-1}$, $\Prit_{\theta}$ denotes the law of $Z_T$ induced by $M = \lambda \theta \theta^\top + \frac{1}{\sqrt{d}}W$. Because this test only requires one bit of information, we need to show that $\TV$-distance between $\Prit_0$ and $\overline{\Prit}$ is $o(1)$ until sufficiently many queries have been made. It will therefore be insufficient to bound the $\TV$-distance between $\overline{\Prit}$ and $\Prit_0$ by the average $\TV$-distance between $\Prit_{\theta}$ and $\Prit_0$ as in the last section, since testing between the two cases with good probability queries only $O(1)$ queries. Hence, we will more carefully bound $\TV$ between $\overline{\Prit}$ and $\Prit_0$ by applying Pinsker's inequality~\cite{tsybakov2009introduction} to the null and mixture distributions. As in the previous section, we will also need to use truncation. 
\begin{prop}[Truncated $\chi^2$ Inequality]\label{truncChi_detec} Let $\calP$ be a distribution over a space $\Theta$, $\{\Prit_{\theta}\}_{\theta \in \Theta}$ be a family of probability measures on $(\calX,\calF)$. Define the marginal distribution $\overline{\Prit}$  on $(\calX,\calF)$ and its restriction $\overline{\Prit}[B;A_{\theta}]$ via
\begin{eqnarray}
\overline{\Prit}[B] = \Exp_{\theta \sim \calP}\Prit_{\theta}[B], & \text{and} & \overline{\Prit}\left[B;A_{\theta}^k\right] = \Exp_{\theta \sim \calP}\Prit_{\theta}\left[B \cap A_{\theta}^k\right].
\end{eqnarray}
Then, if $p = \overline{\Prit}[\calX;A_{\theta}^k]$, we have for any probability measure $\Qit$ on $(\calX,\calF)$
\begin{eqnarray}
\|\Qit - \overline{\Prit}\|_{\TV} \le \frac{1}{2}\sqrt{\Exp_{\Qit}\left[\left(\frac{\mathrm{d}\overline{\Prit}[\cdot;A_{\theta}^k]}{\mathrm{d}\Qit}\right)^{2}\right] -1 } + \frac{\sqrt{2(1-p)}+ (1-p)}{2}.
\end{eqnarray}
\end{prop}
\begin{rem} As in Proposition~\ref{chi_sq_fano_prop}, dispensing with the normalization constants by considering truncated rather than conditional distributions greatly simplifies the presentation. Unlike Proposition~\ref{chi_sq_fano_prop}, the Proposition~\ref{truncChi_detec} is established elementarily, without $f$-divergence machinery.
\end{rem}
As alluded to above, we will take $\Qit = \Prit_{0}$, and $\Prit_{\theta}$ to be the law of $Z_T$ under $M = \lambda \theta \theta^\top + \frac{1}{\sqrt{d}}W$. We take $A_{\theta} = A_{\theta}^T$ defined in Equation~\eqref{Truncated_Distributions}. By definition of $\overline{\Prit}[\cdot;\{A_{\theta}\}]$ and Fubini's theorem, we write
\begin{eqnarray}\label{Equation_Prod_Chi}
\Exp_{\Qit}\left[\left(\frac{\mathrm{d}\overline{\Prit}[\cdot;A_{\theta}]}{\mathrm{d}\Qit}\right)^{2}\right] = \Exp_{\theta,\theta' \sim \calP}\Exp_{\Prit_0}\left[\frac{\mathrm{d}\Prit_{\theta}[\cdot;A_{\theta}^T]\mathrm{d}\Prit_{\theta'}[\cdot;A_{\theta'}^T]}{(\mathrm{d}\Prit_0)^2}\right].
\end{eqnarray}
Equation~\eqref{Equation_Prod_Chi} is a standard observation in combinatioral hypothesis testing~\cite{addario2010combinatorial}, and known as the (conditional) second moment method in probabilistic combinatorics~\cite{achlioptas2006random}. The quantity on the right hand side of Equation~\eqref{Equation_Prod_Chi} resembles the $\chi^2+1$ divergence, but with a product of two different likelihoods in the numerator. Because these likelihoods take large values on different parts of the space, Equation~\eqref{Equation_Prod_Chi} is typically much smaller than the $\chi^2+1$ divergence between $\Prit_{\theta}$ and $\Prit_0$. We now specialize Lemma~\ref{Generic_UB_LL} in the appendix to establish a bound on Equation~\eqref{Equation_Prod_Chi}.
\begin{lem}[Upper Bound on Likelihood Ratios] \label{Detection_LL_UB} Given $u_1,u_2 \in \calS^{d-1}$, and let $A_{u_1}^k$ and $A_{u_2}^k$ be events as in Equation~\eqref{Truncation_Def_Eq}. For $i \ge 1$ and $v^{(1)},\cdots,v^{(i)} \in \calS^{d-1}$, define the expected conditional likelihoods ratios
\begin{eqnarray}
g_i\left(u_1,u_2;\{\vj\}_{1 \le j \le i}\right) &:=& \Exp_{\Prit_0}\left[\frac{\mathrm{d}\Prit_{u_1}(Z_i |Z_{i-1})\mathrm{d}\Prit_{u_2}(Z_i |Z_{i-1})}{\left(\mathrm{d}\Prit_0(Z_i | Z_{i-1}\right))^2 } \Big{|} \{\vj\}_{1 \le j \le i}\right].
\end{eqnarray}
Moreover, let $\calV_{\theta}^k = \left\{\vone,\dots,\vk \in (\calS^{d-1})^{k}~:~ \forall i \in [k],~d \cdot \langle \vi, \theta \rangle^2 \le \tau_i\right\}$. Then, 
\begin{eqnarray}
\Exp_{\Prit_0}\left[\frac{\mathrm{d}\Prit_{u_1}(Z_{k} ; A_{u_1}^k)\mathrm{d}\Prit_{u_2}(Z_{k} ; A_{u_2}^k)}{\left(\mathrm{d}\Prit_0(Z_{k})\right)^2} \right] \le \sup_{\vone,\dots, \vk \in \calV_{u_1}^k \cap \calV_{u_2}^k} ~ \prod_{i=1}^k g_i\left(u_1,u_2;\{\vj\}_{1 \le j \le i}\right).
\end{eqnarray}
\end{lem}
With a bit of computation, one can make the above bound more explicit (see Proposition~\ref{Inner_product_constraints_cor}):
\begin{eqnarray}\label{second_moment_computation}
 \Exp_{\Prit_0}\left[\frac{\mathrm{d}\Prit_{u_1}(Z_T ; A^T_{u_1})\mathrm{d}\Prit_{u_2}(Z_T;A^T_{u_2})}{\left(\mathrm{d}\Prit_0(Z_T )\right)^2}  \right] \le e^{\lambda^2|\langle u_1 , u_2 \rangle| \sum_{i=1}^T \tau_i + \frac{\lambda^2}{d}(\sum_{i=1}^T \tau_i)^2}.
\end{eqnarray}
To wrap up, we will again take $\calP$ to be the uniform measure on the sphere, which satisfies the following moment bound.
\begin{lem}\label{Sphere_MGF} Let $\calP$ be the uniform measure on the sphere $\calS^{d-1}$. Then for all $\lambda \ge 0$ and $v \in \calS^{d-1}$,
\begin{eqnarray}
\Exp_{\theta \sim \calP}\left[e^{\lambda|\langle \theta, v \rangle|}\right] \le e^{4\lambda^2/d+\lambda\sqrt{2/d}}.
\end{eqnarray}
\end{lem}
\begin{proof}[Proof of Lemma~\ref{Sphere_MGF}] Write $\Exp[e^{\sqrt{d}\lambda|\langle \theta, v \rangle|}] = e^{\lambda\sqrt{2}}\cdot\Exp[e^{\lambda(\sqrt{d}|\langle \theta, v \rangle| - \sqrt{2}\lambda)}]$. Since $\sqrt{d}(|\langle \theta, v \rangle| - \sqrt{2})$ satisfies the sub-Gaussian tail bound from Lemma~\ref{SphereConcentation}, a standard conversion from tail bounds to moment generating functions (see, e.g., Lemma 1.5 in \cite{rigollet201518}) yields that
$\Exp[e^{\lambda\cdot \sqrt{d}(|\langle \theta, v \rangle| - \sqrt{2})}] \le e^{4\lambda^2}$. Replacing $\lambda$ with $\lambda/\sqrt{d}$ concludes the proof.
\end{proof}

Now, we take an expectation of the bound in Equation~\eqref{second_moment_computation} over $u_1,u_2 \sim \calP$, and apply the sub-Gaussian bound of Lemma~\ref{Sphere_MGF}. Further, appropriately choosing the parameters $\tau_1,\cdots,\tau_T$, and applying Proposition~\ref{truncChi_detec} leads to the following result. 
\begin{prop}\label{Detection_Prop} There exists a absolute constant $c_1 >1$ such that for all $\lambda > 2$ we have
\begin{eqnarray}
\|\Prit_0 - \overline{\Prit}\|_{\TV} \le \frac{\sqrt{2}\left(c_1\lambda\right)^{T}}{d^{1/4}}\cdot \left(\sqrt{\log\frac{d}{\left(c_1\lambda\right)^T}} + 4\right).
\end{eqnarray}
\end{prop}
The proof of the above Proposition is detailed in Appendix~\ref{DetecLemProof}. We now state and prove our main technical result for detecting a large eigenvalue.
\begin{thm}[Technical Statement of Theorem~\ref{MainTheoremDetection}]\label{Main_Formal_Detection_Theorem} For any $\delta_0 \in (0,1)$ and $\lambda \ge K_d + 4d^{-1/2}\sqrt{\log(1/\delta_0})$, there exist two distributions $\calD_0$ and $\calD_1$ on $\Sym^{d \times d}$ such that 
\begin{enumerate}
	\item $\|M\| \le K_d + 2d^{-1/2}\sqrt{\log(1/\delta_0)}$ $\calD_0$-almost surely, 
	\item $\|M\| \ge \lambda - 2d^{-1/2}\sqrt{\log(1/\delta_0)}$ $\calD_1$-almost surely,
\end{enumerate}
and, for any $T$-query algorithm which outputs a binary test $\psi \in \{0,1\}$, 
\begin{eqnarray}
\Pr_{M\sim \calD_0,\Alg}[\psi = 1] + \Pr_{M\sim\calD_1,\Alg}[\psi = 0] \ge 1 - \frac{\sqrt{2}\left(c_1\lambda\right)^{T}}{d^{1/4}}\cdot \left(\sqrt{\log\frac{d}{\left (c_1 \lambda\right)^T}} + 4\right) - 3\delta_0.
\end{eqnarray}
Moreover, $\calD_1$-almost surely, $M \in M_{\gamma}$ for $\gamma(\lambda,\delta)$ as defined in Proposition~\ref{Reduction_Theorem}. Here, $c_1$ is the absolute constant of Proposition~\ref{Detection_Prop}.
\end{thm}

\begin{proof}[Proof of Theorem~\ref{Main_Formal_Detection_Theorem}]
Let $\overline{A}$ denote the event of Proposition~\ref{Reduction_Theorem} marginalized over $\theta \sim \calP$. This event holds with probability at least $1 - 2\delta_0$ under $\overline{\Pr}$.  The proof of Theorem~\ref{Main_Formal_Detection_Theorem} also show that, with probability $1-\delta_0$ under $\Pr_0$, $\|W\| \le K_d + 2\sqrt{\log(1/\delta_0)}$; we denote this event $A_0$. Then, we have that
\begin{eqnarray*}
\Pr_{0}[\psi = 1 | A_0] + \overline{\Pr}\left[\psi = 0| \overline{A}\right] 
&\ge& \Pr_0[\psi = 1] + \overline{\Pr}[\psi = 0] - 3\delta_0\\
&\overset{(i)}{\ge}& 1-\|\Prit_0 - \overline{\Prit}\|_{\TV} - 3\delta_0\\
&\overset{(ii)}{\ge}& 1 - \frac{\sqrt{2}\left(c_2 \lambda^2\right)^{T/2}}{d^{1/4}}\cdot \left(\sqrt{\log\frac{d}{\left(\gamma'' c_2\right)^T}} + 4\right) - 3\delta_0.
\end{eqnarray*}
where $(i)$ follows from a standard hypothesis-testing inequality (Theorem 2.2 in~\cite{tsybakov2009introduction}) and the fact that $\psi$ is a measurable function of $Z_T$, and $(ii)$ is Lemma~\ref{Detection_Prop} above. Finally, on $\overline{A}$, $M \in \calM_{\gamma}$ for $\gamma = \gamma(\lambda,\delta)$ defined in Propostion~\ref{Reduction_Theorem}. Finally take $\calD_0 := \Pr_0[\cdot | A_0]$ and $\calD_1 := \overline{\Pr}[\cdot | \overline{A}]$ . 
\end{proof}

\section{Conclusion}
This paper established a fundamental separation between the first-order query complexity for optimizing strict-saddle and truly convex objectives. We demonstrated the separation by establishing a query complexity lower bound of $\Omega(\log d)$ in the easy regime of rank-one PCA where the leading and second eigenvalues are well separated. An exciting direction for future work is to attempt to prove lower bounds for the ``hard'' regime of rank-one PCA, where the gap between the leading  and second eigenvalues, $1-\gamma$, is arbitrarily small. Ideally one would show $\Omega\left(\frac{\log (d/\epsilon)}{\sqrt{1-\gamma}}\right)$ gap-dependent and $\Omega\left(\frac{\log d}{\sqrt{\epsilon}}\right)$ query complexity lower bounds, thereby matching the randomized block-Krylov methods from Musco and Musco~\cite{musco2015randomized}. 

We believe that establishing such a lower bound would entail numerous technical challenges. As remarked at the end of Theorem~\ref{Main_Cor_estimation}, establishing this bound for the deformed Wigner model would requires obtaining sharp control on the rate of information accumulation, at the so-called ``identifiability threshold'', where $\lambda$ tends to $1$\footnote{It is a non-obvious fact from random matrix theory that the deformation is detectable even when $\lambda < \Exp[\|W\|]$. Indeed $\|W + \lambda \theta \theta^{\top}\| $ concentrates around $\lambda + 1/\lambda$ for $\lambda \ge 1$. This implies that $\|W + \lambda \theta \theta^{\top}\| > \|W\|$, even when $\lambda <2$.}.  One would also need to invoke non-trivial machinery from random matrix theory~\cite{feral2007largest} to sharpen Proposition~\ref{Reduction_Theorem} in order to establish that $v_1(M)$ correlates with $\theta$ when $\lambda \in [1,K_d]$.

Another direction for future research would be to understand if the classical lower bounds against Krylov methods for \emph{convex} optimization~\cite{nemirovskii1983problem} hold in the stronger gradient-query model as well. This would resolve objections raised in the literature that these lower bounds place unduly strong assumptions on the class of gradient-methods considered, and in particular, do not apply when the starting point of the algorithm may be randomized. More broadly, understanding the query complexity of both non-convex and convex optimization would serve to elucidate the connections between sequential optimization and adaptive estimation.

\section*{Acknowledgements}
We thank Chi Jin and Darren (Tianyi) Lin for their immensely useful feedback. Max Simchowitz is supported by an NSF GRFP  fellowship. Benjamin Recht is generously supported by NSF award CCF-1359814, ONR awards N00014-14-1-0024 and N00014-17-1-2191, the DARPA Fundamental Limits of Learning Program, a Sloan Research Fellowship, and a Google Faculty Award. 

\bibliographystyle{plain}
\bibliography{PCA}
\clearpage

\appendix

\section{Proof of Proposition~\ref{Reduction_Theorem}}
In this section, we prove the validity of our reduction. To prove the strongest lower bounds, we want to allow $\lambda$ to be arbitrarily close to $\lambda \approx K_d = 2 + o(1)$. 
This precludes using coarser arguments, e.g., bounding the $\|W\|_{op}$ and applying the Davis-Kahan Sine theorem. We start with a deterministic result which sharpens Davis-Kahan in our particular case of interest:
\begin{lem}\label{InnerProdLem} Let $M = \lambda \theta \theta^\top + W$, where $\theta \in \calS^{d-1}$ and $W$ is an arbitrary symmetric matrix. Set $\gamma := \frac{\|W\|}{\theta^\top M\theta}$. Then, for any $\epsilon \le \gamma$, if $w^\top Mw \ge (1-\epsilon)\theta^\top M\theta$, then 
\begin{eqnarray*}
|\langle w, \theta \rangle| \ge F(\epsilon,\gamma) := \sqrt{\left(\frac{\gamma}{1-\gamma}\right)^2 + 1 - \frac{\epsilon}{1-\gamma}} - \frac{\gamma}{1-\gamma}.
\end{eqnarray*}
Moroever, $\lambda_1(M) \ge \theta^\top M\theta$, and $\lambda_i(M) \le \|W\|$ for all $i \ge 2$.
\end{lem}
We now apply above lemma with $W/\sqrt{d}$. To conclude the proof of Proposition~\ref{Reduction_Theorem}, it suffices to verify that the following two conditions hold with probability $1-\delta$,
\begin{eqnarray}\label{Reduction_Want_To_Show}
\|W\| \le K_d + 2\sqrt{\log(1/\delta)/d} & \text{and} & \theta^\top M\theta - \lambda = \theta^\top W\theta \ge -2\sqrt{\log(1/\delta)/d}.
\end{eqnarray}
To this end, we use an alternate characterization of the $\GOE$:
\begin{lem}  $W\sim \GOE(d)$ has the distribution $\frac{1}{\sqrt{2}}(X + X^{\top})$, where $X \in \R^{d\times d}$ is a matrix with i.i.d.\ $\calN(0,1)$ entries.
\end{lem}
\begin{proof} Let $\widetilde{W} := \frac{1}{\sqrt{2}}(X + X^{\top})$, where $X$ has i.i.d.\ standard normal entries. Then $\widetilde{W}$ is symmetric, the entries $\widetilde{W}_{i,j}$ for $i \le j$ are independent centered Gaussians with $\Exp[\widetilde{W}^2_{ii} = \Exp[(\sqrt{2}X_{ii})^2] = 2$, while for $i \ne j$, $\Exp[\widetilde{W}^2_{ii} =\frac{1}{2} \Exp[(X_{ij}+X_{ji})^2] = 1$ since $X_{ij}$ and $X_{ji}$ are independent. 
\end{proof}
Note then that if $X \in \R^{d\times d}$ is a matrix with i.i.d.\ $\calN(0,1)$, then  $f_1(X) := \theta^{\top}\frac{1}{\sqrt{2}}(X + X^{\top})\theta$ has the same distribution as $\theta^{\top}W\theta$, and $f_2(X) := \|\frac{1}{\sqrt{2}}(X + X^{\top})\|$ has the same distribution as $\|W\|$. The following lemma shows that $f_1$ and $f_2$ are both $\sqrt{2}$-Lipschitz. 

\begin{lem} Let $v$ be a unit vector. Then the mappings $X \mapsto v^\top \frac{1}{\sqrt{2}}(X + X^\top)v$, $X \mapsto \|\frac{1}{\sqrt{2}}(X + X^\top )\|$ are $\sqrt{2}$-Lipschitz. 
\end{lem}
\begin{proof}
By Cauchy Schwartz, we have for any $v,w \in \calS^{d-1}$ 
\begin{eqnarray*}
|v^{\top}\frac{1}{\sqrt{2}}(X+X^{\top})w| = \sqrt{2}|v^{\top}Xw| \le \sqrt{2}\cdot\|X\|_{F}\|wv^{\top}\|_F = \sqrt{2}\|X\|_F.
\end{eqnarray*}
Since the Lipschitz constant of a linear map is equal to its operator norm, we have that $x \mapsto v^{\top}\frac{1}{\sqrt{2}}(X+X^{\top})w$ is $\sqrt{2}$-Lipschitz. The first part of the lemma follows by taking $v = w$. The second  point follows by noting that $\|\frac{1}{\sqrt{2}}(X+X^{\top})\|_{2} = \sup_{v,w \in \calS^{d-1}} v^{\top}\frac{1}{\sqrt{2}}(X+X^{\top})w$, and the supremum of a collection of $L$-Lipschitz functions is $L$-Lipschitz.
\end{proof}
Equation~\eqref{Reduction_Want_To_Show} now follows by applying the following concentration inequality to the functions $f_1(X)$ and $f_2(X)$: 
\begin{lem}[Tsirelson-Ibgragimov-Sudakov, Theorem 5.5 in~\cite{boucheron2013concentration}] Let $f$ be a $L$-Lipschitz function and let $X$ be a standard Gaussian vector. Then,
\begin{eqnarray*}
\Pr[f(X) \ge \Exp[f(X)] + t] \vee \Pr[f(X) - \Exp[f(X)] \le - t] \le e^{-t^2/2L^2}.
\end{eqnarray*}
\end{lem}

\subsection{Proof of Lemma~\ref{InnerProdLem}}
Recall that $\theta \in \calS^{d-1}$ Any vector $w \in \calS^{d-1}$ can be written as $\alpha \theta + \sqrt{1-\alpha^2}v$, where $v \perp \theta$ and $v\in \calS^{d-1}$. Note then that $\alpha = \langle w, \theta \rangle$. By replacing $v$ with $-v$, we may assume without loss of generality that $\alpha \ge 0$. For ease of notation, set $K_2 = \theta^\top M\theta$ and  $K_1 = \|W\|$. Then, $w^{\top}Mw$ can be written as
\begin{eqnarray*}
(\alpha \theta +  \sqrt{1-\alpha^2}v)^{\top}M(\alpha \theta +  \sqrt{1-\alpha^2}v) &=& \alpha^2 \theta^{\top}M\theta + (1-\alpha^2)v^{\top}Mv + 2\alpha\sqrt{1-\alpha^2}v^{\top}M\theta\\
&\le& \alpha^2 \theta^{\top}M\theta + (1-\alpha^2)K_1 + 2|\alpha| K_1.
\end{eqnarray*}
Thus, if $w^{\top}Mw \ge (1-\epsilon)\|M\| \ge (1-\epsilon)K_2$, then
\begin{eqnarray*}
(1 - \epsilon - \alpha^2)K_2 - (1-\alpha^2)K_1 - 2\alpha K_1 \le 0.
\end{eqnarray*}
Letting $\gamma = K_1/K_2$, we have
\begin{eqnarray*}
\alpha^2 (1- \gamma)  - (1 - \epsilon - \gamma) + 2\alpha \gamma \ge 0.
\end{eqnarray*}
Thus, solving the quadratic inequality and taking the positive part,
\begin{eqnarray*}
\alpha &\ge& \frac{-2\gamma + \sqrt{4\gamma^2 + 4(1-\gamma)(1-\epsilon - \gamma)}}{2(1-\gamma)}\\
&=&  \sqrt{(\frac{\gamma}{1-\gamma})^2 + 1 - \frac{\epsilon}{1-\gamma}} - \frac{\gamma}{1-\gamma}.
\end{eqnarray*}

The second statement in Lemma~\ref{InnerProdLem} follows since $\lambda_1(M) \ge \theta^\top  M \theta  = \lambda + \theta^\top W\theta$ and $\lambda \ge 0$, and the third statement follows from eigenvalue interlacing (e.g., Corollary 4.3.9 in~\cite{horn2012matrix}).

\subsection{Behavior of $F(\epsilon,\gamma)$\label{f_fact_proof}}

The following lemma describes the behavior of $F(\epsilon,\gamma)$ is both the small- and large-$\gamma$ regimes:
\begin{lem}
$\lim_{\gamma \to 0}F(\epsilon,\gamma) = \sqrt{1 - \epsilon}$, and for any $\epsilon \in (0,1-\gamma)$,
\begin{eqnarray}
F(\epsilon,\gamma) \ge \frac{1}{2\sqrt{2}}\min \left\{\sqrt{1 - \frac{\epsilon}{1-\gamma}}, \frac{1 - \gamma - \epsilon}{\gamma}\right\}.
\end{eqnarray}
\end{lem}
\begin{proof} 
The computation of $\lim_{\gamma \to 0}F(\epsilon,\gamma)$ is clear from the definition of $F(\epsilon,\gamma)$. For the second statement, we see that Taylors theorem implies $\sqrt{a^2 + x} - a \ge \frac{x}{2\sqrt{a^2+x}} \ge \frac{x}{2\sqrt{2}}\min\{1/a,1/\sqrt{x}\}$. This entails
\begin{eqnarray*}
F(\epsilon,\gamma) &\ge& \frac{1}{2}(1 - \frac{\epsilon}{1-\gamma}) \cdot \left(\sqrt{\left(\frac{\gamma}{1-\gamma}\right)^2 + 1 - \frac{\epsilon}{1-\gamma}}\right)^{-1}\\
&\ge& \frac{1}{2\sqrt{2}}\left(1 - \frac{\epsilon}{1-\gamma}\right) \cdot \left(\max\left\{\left(\frac{\gamma}{1-\gamma}\right)^2,  1 - \frac{\epsilon}{1-\gamma}\right\}\right)^{-1/2}\\
&=& \frac{1}{2\sqrt{2}}\left(1 - \frac{\epsilon}{1-\kappa}\right) \cdot \min \left\{\frac{1-\gamma}{\gamma},  \left(1 - \frac{\epsilon}{1-\gamma}\right)^{-1/2}\right\}\\
&=& \frac{1}{2\sqrt{2}}\min \left\{\sqrt{1 - \frac{\epsilon}{1-\kappa}}, \frac{1 - \gamma - \epsilon}{\gamma}\right\}. 
\end{eqnarray*}
\end{proof}
\section{Proof Lemma~\ref{ConditionalLemma}\label{CondLemProof}}
Recall the definition $\Sigma_i := P_{i-1}(I_d+v^{(i)}v^{(i)\top})P_{i-1}$, and that $\mathcal{F}_{i-1}$ is the $\sigma$-algebra generated by $\vone,\wonetil,\dots,\viminustil,\wiminustil$. Since our algorithm is deterministic, $\vi$ is $\calF_{i-1}$ measurable. It then suffices to show that 
	\begin{eqnarray}
	\wjtil = P_{j-1}W\vj \big{|}\calF_{i} \sim \mathcal{N}(0,\Sigma_i)~.
	\end{eqnarray}
	Recall from Section~\ref{Likelihood_Comp} that $\Sigma_i$ is degenerate, so we understand $\mathcal{N}(0,\frac{1}{d}\Sigma_i)$ as a normal distribution absolutely continuous with respect to the Lebesque measure supported on $(\ker P_{i-1})^{\perp}$. Note that $\widetilde{w}^{(i)}$ is conditionally independent of $\wonetil,\dots,\wiminustil$ given $\vone,\dots,\viminus,\vi$. Consequently, the conditional distribution of $\witil$ given $\mathcal{F}_i$ can be computed as if the queries $\vone,\dots,\vi$ were fixed in advanced. 

	Hence, throughout, we shall assume that $\vone,\dots,\vi$ are deterministic, and consider the joint distribution of $\wonetil,\dots,\wiminustil,\witil$. We will show that $\witil$ is independent of $\wonetil,\dots,\wiminustil$, and that its marginal is $\mathcal{N}(0,\frac{1}{d}\Sigma_i)$. Since the map $W \mapsto P_{j-1}W\vj$ is linear maps, $\wonetil,\dots,\witil$ are jointly Gaussian with mean zero. Thus, it suffices to show that 1) the (marginal) covariance of $\witil$ is $\Sigma_i$ and, 2) the covariance between $\witil$ and $\wjtil$ for $j \ne i$ is $0$. The covariances are computed as
	\begin{eqnarray}\label{CovEq}
	\Exp\left[\widetilde{w}^{(j)} \widetilde{w}^{(j)\top}\right] = \Exp\left[(P_{i-1}Wv^{(i)})(P_{j-1}Wv^{(j)})^{\top}\right] = P_{i-1} \Exp\left[Wv^{(i)}v^{(j)\top}W\right] P_{j-1}~,
	\end{eqnarray}
	and we compute the inner term with the following lemma.
	\begin{lem}\label{GaussComp}
	For any $v^{(i)},v^{(j)}$, one has
	\begin{eqnarray}\label{CovEq2}
	\Exp\left[Wv^{(i)}v^{(j)\top}W\right] = v^{(j)}v^{(i)\top} + \langle v^{(i)},v^{(j)} \rangle I ~.
	\end{eqnarray}
	\end{lem} 
	For $\vi = \vj$, Equations~\eqref{CovEq} and~\eqref{CovEq2} immediately imply $\witil$ has covariance $\Sigma_i$. Moreover, for $j < i$, we have
	\begin{eqnarray}
	P_{i-1} \Exp\left[Wv^{(i)}v^{(j)\top}W\right] P_{j-1} \overset{(i)}{=} P_{i-1}v^{(i)}v^{(j)\top}P_{j-1} \overset{(ii)}{=} 0
	\end{eqnarray}
	where $(i)$ holds from Lemma~\ref{GaussComp} and the fact that $\langle \vi,\vj \rangle = 0$ (since $\vone,\dots,\vi$ are assumed to be orthogonal), and $(ii)$ holds since $\vj \in \ker(P_{i-1})$, as $P_{i-1}$ projects onto the complement of $\{\vone,\dots,\viminus\}$.

	\begin{proof}[Proof of Lemma~\ref{GaussComp}] 
	For $a \in \{1,\cdots,d\}$,
	\begin{eqnarray*}
	&& \Exp[Wv^{(i)}v^{(j)\top}W]_{aa} = \Exp[W_{aa}^2]v_a^{(i)}v^{(j)}_a + \sum_{p \ne a}\Exp[W_{ap}^2]v_a^{(i)}v^{(j)}_a \\
	&=& 2v_a^{(i)}v^{(j)}_a + \sum_{p \ne a}v_p^{(i)}v^{(j)}_p = v_a^{(i)}v^{(j)}_a + \sum_{p}v_p^{(i)}v^{(j)}_p = v_a^{(i)}v^{(j)}_a + \langle v^{(i)}, v^{(j)}\rangle ~.
	\end{eqnarray*}
	Whereas for $a \ne b$, 
	\begin{eqnarray}
	\Exp[Wv^{(i)}v^{(j)\top}W]_{ab} &=& \Exp\left[\sum_{p,q}W_{ap}W_{bq}v_p^{(i)}v^{(j)}_q\right] = \sum_{p,q}\Exp[W_{ap}W_{bq}]v_p^{(i)}v^{(j)}_q~.
	\end{eqnarray}
	$W_{ap}$ and $W_{bq}$ are independent unless $(a,p) = (b,q)$ or $(a,p) = (q,b)$. If $a \ne b$,  then this means the only term in the above sum which is non zero is $p = b$ and $q = a$, which yields $v_a^{(j)}v_b^{(i)}$.
	\end{proof}

\section{Information-Theoretic Computations}

\subsection{Supporting proofs for Section~\ref{first_attempt_sec}\label{Sec:WarmupAppend}}

We begin by stating a well-known computation (see, e.g., \cite{cover2012elements}) for $\KL$ divergence between two normal distribution:
\begin{lem}\label{KLlem1} Let $X \sim \mathcal{N}(\mu_X,\Sigma)$ and $Y \sim \mathcal{N}(\mu_Y,\Sigma)$, where $\Sigma \succeq 0 $ and $\mu_X - \mu_Y \in (\ker \Sigma)^{\perp}$, then 
\begin{eqnarray}
\KL(\Prit_{X},\Prit_{Y}) = \frac{1}{2}\|\mu_X - \mu_Y\|^2_{\Sigma}.
\end{eqnarray}
\end{lem}
We now compute the mutual information between $Z_k$ and $\theta$, thereby proving Proposition~\ref{MutualInfoProp}:
\begin{proof}[Proof of Proposition~\ref{MutualInfoProp}]
Recall that $\Prit_{u_0}$ and $\Prit_{u_1}$ denote the restriction of the measures $\Pr_{u_0}$ and $\Pr_{u_1}$ to events which are measurable with respect to $Z_k$. We then compute
\begin{eqnarray*}
\KL(\Prit_{u_0},\Prit_{u_1}) &=& \Exp_{\Prit_{u_0}}\log \frac{\rmd\Prit_{u_0}(Z_k)}{\rmd\Prit_{u_1}(Z_k) }\\
&=& \Exp_{\Prit_{u_0}}\sum_{i=1}^k \log \frac{\rmd\Prit_{u_0}((\wi,\vi) | Z_{i-1})}{\rmd\Prit_{u_1}((\wi,\vi) | Z_{i-1}) }\\
&=& \sum_{i=1}^k \Exp_{\Prit_{u_0}}\Exp_{\Prit_{u_0}}\left[ \log \frac{\rmd\Prit_{u_0}(\wi | Z_{i-1})}{\rmd\Prit_{u_1}(\wi | Z_{i-1}) } \Big| Z_{i-1}\right]\\
&\overset{(i)}{=}& \sum_{i=1}^k \Exp_{\Prit_{u_0}}\left[\KL\left(\calN\left(u_0^\top \vi\lambda P_{i-1} u_0,\frac{1}{d}\Sigma_i\right),\calN\left(u_1^\top \vi\lambda P_{i-1} u_1,\frac{1}{d}\Sigma_i\right)\right)\right].
\end{eqnarray*}
where $(i)$ uses the fact that $\vi$ is $Z_{i-1}$-measurable, and uses the computation of the conditional distribution of $\wi$ from Lemma~\ref{ConditionalLemma}. We then compute 
\begin{align}\label{KL_warmup}
\KL\Big(\calN\Big(u_0^\top \vi\lambda P_{i-1} u_0,\frac{1}{d}\Sigma_i\Big), & ~ \calN\Big(u_1^\top \vi\lambda P_{i-1} u_1,\frac{1}{d}\Sigma_i\Big)\Big) \nonumber\\
 &\overset{(i)}{=} \frac{1}{2}\left\| u_0^\top \vi\lambda P_{i-1} u_0 - u_1^\top \vi\lambda P_{i-1} u_1\right\|_{\frac{1}{d}\Sigma_i}^2 \nonumber\\
&\overset{(ii)}{=} \frac{\lambda^2 d}{2}\left\|P_{i-1} ( u_0^\top \vi u_0 - u_1^\top \vi u_1)\right\|_{\Sigma_i}^2 \nonumber \\
&\overset{(iii)}{\le} \frac{\lambda^2 d}{2} \left\| u_0^\top \vi u_0 - u_1^\top \vi u_1\right\|^2 \nonumber \\
&= \frac{\lambda^2 d}{2} \left(\langle u_0, \vi\rangle^2 + \langle u_1,\vi \rangle^2  - 2 \langle u_0, \vi \rangle \langle u_1, v^{(i)} \rangle \langle u_0, u_1 \rangle \right),
\end{align}
where $(i)$ uses Lemma~\ref{KLlem1}, and $(ii)$ and $(iii)$ use Equation~\eqref{InverseInequality} from Section~\eqref{Likelihood_Comp}. Hence, if $\calP$ is isotropic, then $\Exp_{\theta \sim \calP}[\theta \theta^\top ] = \frac{1}{d}I$, so we can bound
\begin{eqnarray*}
I(Z_k;\theta) &\overset{(i)}{\le}& \Exp_{u_0,u_1 \sim \calP}\KL(\Prit_{u_0},\Prit_{u_1})\\
&\overset{(ii)}{\le}& \frac{\lambda^2 d}{2}\Exp_{u_0,u_1 \sim \calP}\sum_{i=1}^k \Exp_{\Prit_{u_0}}\left[\langle u_0, \vi\rangle^2 + \langle u_1,\vi \rangle^2 - 2v^{(i)\top} u_1u_1^\top u_0u_0^\top \vi\right]\\
&\overset{(iii)}{=}& \frac{\lambda^2 d}{2}\Exp_{u_0\sim \calP}\sum_{i=1}^k \Exp_{\Prit_{u_0}}\Exp_{u_1 \sim \calP}\left[\langle u_0, \vi\rangle^2 + \langle u_1,\vi \rangle^2 - 2v^{(i)\top} u_1u_1^\top u_0u_0^\top \vi\right]\\
&\overset{(iv)}{=}& \frac{\lambda^2 d}{2}\Exp_{u_0\sim \calP}\sum_{i=1}^k \left((1 - \frac{2}{d})\Exp_{\Prit_{u_0}}\left[\langle u_0, \vi\rangle^2\right] + \frac{1}{d}\right)\\
&=& \frac{\lambda^2}{2}\left\{k + (1-\frac{2}{d}) \Exp_{u_0 \sim \calP}\Exp_{\Prit_{u_0}}\left[\sum_{i=1}^kd\langle u_0, \vi\rangle^2 \right]\right\}.
\end{eqnarray*}
where $(i)$ uses convexity of the KL divergence~\cite{cover2012elements}, $(ii)$ is Equation~\eqref{KL_warmup}, $(iii)$ is Fubini's theorem, and $(iv)$ uses the fact that $\calP$ is centered and isotropic, whereas $\vi$ are unit vectors.
\end{proof}
Finally, we establish the mutual information recursion from Theorem~\ref{KL_estimation_theorem}.
\begin{proof}[Proof of Theorem~\ref{KL_estimation_theorem}]
The first equation of the theorem, Equation~\ref{KL_Recursion}, follows directly from Proposition~\ref{MutualInfoProp}, Proposition~\ref{Global_Fano}, and the choice of loss function $\calL(\vkplus,\theta) = \I(\langle \vkplus,\theta \rangle \ge \tau)$. By integrating Equation~\ref{KL_Recursion}, we have for any $\tau^*$
\begin{eqnarray*}
\tau_{k+1} &\le& \tau^* + \int_{t = \tau^*}^{d} \left(\frac{\log 2 + \frac{\lambda^2}{2} (k+\sum_{i=1}^k\tau_i)}{C_1t-C_2}\right)dt\\
&=& \tau^* +  \frac{\log 2 + \frac{\lambda^2}{2} (k+\sum_{i=1}^k\tau_i)}{C_1}\log \frac{d}{\tau^* - C_2/C_1}~.
\end{eqnarray*}
Setting $\tau^* = C_2/C_1 + \frac{\log 2 + \frac{\lambda^2}{2} (k+\sum_{i=1}^k\tau_i)}{C_1}$ yields 
\begin{eqnarray*}
\tau_{k+1} &\le& \frac{C_2}{C_1} + \frac{\log 2 + \frac{\lambda^2}{2} (k+\sum_{i=1}^k\tau_i)}{C_1}\left(1 + \log \frac{C_1d}{\log 2 + \frac{\lambda^2}{2} (k+\sum_{i=1}^k\tau_i)}\right)~.
\end{eqnarray*}
\end{proof}

\subsection{General Upper Bound on Likelihood Second Moments~\label{Chi_sq_proof}}
Again, fix $\tau_1,\dots,\tau_{T+1}$, and define the event $A_\theta^{k} := \{\forall i \in [k], ~d\langle \vi, \theta \rangle^2 \le \tau_i\}$. We now proving the following lemma, which ecompasses both Propositions~\ref{Estimation_LL_UB} and~\ref{Detection_LL_UB}:

\begin{lem}[Generic Upper Bound on Likelihood Ratios] \label{Generic_UB_LL}Let $u,s\in \calS^{d-1}$.  Define the expected likelihood ratios:
\begin{eqnarray}\label{g_func_def}
g_i(u,s;\{v^{(j)}\}_{1 \le j \le i} ) &:=& \Exp_{\Prit_0}\left[\frac{\rmd\Prit_u(Z_i |Z_{i-1})\cdot \rmd\Prit_s(Z_i |Z_{i-1})}{\left(\rmd\Prit_0(Z_i | Z_{i-1}) \right)^2} \big{|} \{v^{(j)}\}_{1 \le j \le i}\}\right].
\end{eqnarray}
Then, letting $\calV_{\theta}^k := \{\vone,\dots,\vk: \forall i \in [k], ~ d\langle \vi, \theta \rangle^2 \le \tau_i\} $, we have
\begin{eqnarray}\label{product_eq}
\Exp_{\Prit_0}\left[\frac{\rmd\Prit_u(Z_k ; A_u^k)\cdot \rmd\Prit_s(Z_k;A_s^k)}{\left(\rmd\Prit_0(Z_k) \right)^2} \right] &\le& \sup_{\vone,\dots,\vk \in \calV_{u}^k \cap \calV_s^k}\prod_{i=1}^kg_i(u,s;\{v^{(j)}\}_{1 \le j \le i}).
\end{eqnarray}
\end{lem}

	\begin{proof}[Proof of Lemma~\ref{Generic_UB_LL}] 	
	Introduce the shorthand $v^{(1:i)} = \{v^{(j)}\}_{1 \le j \le i}$, with the convention $v^{(1:0)} = \emptyset$. Fixing $u$ and $s$,  $A^j = A^{j}_{u} \cap A^j_s$ for $1 \le j \le k$, with the convention $A^0$ is the entire event space. We also define the ``tail set''
	\begin{eqnarray}
	\calV^k_{j+1}(Z_{j-1}) := \{v^{(j+1)},\dots,\vk: \forall i \in \{j+1,k\}, \theta \in \{u,s\}: d\langle \vi, \theta \rangle^2 \le \tau_i\}.
	\end{eqnarray}
	Finally, we define a partial supremum over the products of the terms $g_i(u,s;\{v^{(j)}\}_{1 \le j \le i} )$: 
	 \begin{eqnarray}
	 G_{j}(v^{(1:j)}) = \sup_{\vjplustil,\dots,\vktil \in \calV^k_{j+1}(Z_{j-1})} \prod_{i=j+1}^k g_i(u,s;\{v^{(1:j)},\widetilde{v}^{(j+1:i)}\}).
	 \end{eqnarray}
	 adopting the convention $ G_{k}(v^{(1:k)}) = 1 $. We observe also that
	 \begin{eqnarray}
	  G_{0}(v^{(1:0)}) =  \sup_{\vonetil,\dots,\vktil \in \calV_u^k \cap \calV_s^k} \prod_{i=1}^k g_i(u,s;\widetilde{v}^{(1:k)}).
	 \end{eqnarray}
	 For any $j \in [k]$, we compute
		\begin{eqnarray*}\label{Main_Inductive_Step_Chi}
		&&\Exp_{\Prit_0}\left[G_{j}(v^{(1:j)}) \frac{\rmd\Prit_u(Z_j)\cdot \rmd\Prit_s(Z_j)}{\left(d\Prit_0(Z_j) \right)^2} \I(A^{j} )\right] \\
		&=&\Exp_{\Prit_0}\left[\Exp\left[G_{j}(v^{(1:j)}) \frac{\rmd\Prit_u(Z_j)\cdot \rmd\Prit_s(Z_j)}{\left(\rmd\Prit_0(Z_j) \right)^2} \I(A^{j} ) \big{|} Z_{j-1} \right]\right] \\
		&=&\Exp_{\Prit_0}\left[\Exp\left[G_{j}(v^{(1:j)}) \frac{\rmd\Prit_u(Z_{j-1})\cdot \rmd\Prit_s(Z_{j-1})}{\left(\rmd\Prit_0(Z_{j-1}) \right)^2} \cdot \frac{\rmd\Prit_u(Z_j|Z_{j-1})\cdot \rmd\Prit_s(Z_j|Z_{j-1})}{\left(\rmd\Prit_0(Z_j|Z_{j-1}) \right)^2} \I(A^{j} ) \big{|} Z_{j-1} \right]\right] \\
		&\overset{(i)}{=}&\Exp_{\Prit_0}\left[\frac{\rmd\Prit_u(Z_{j-1})\cdot \rmd\Prit_s(Z_{j-1})}{\left(\rmd\Prit_0(Z_{j-1}) \right)^2} \cdot \Exp\left[G_{j}(v^{(1:j)}) \frac{\rmd\Prit_u(Z_j|Z_{j-1})\cdot \rmd\Prit_s(Z_j|Z_{j-1})}{\left(\rmd\Prit_0(Z_j|Z_{j-1}) \right)^2} \I(A^{j} ) \big{|} Z_{j-1} \right]\right] \\
		&\overset{(ii)}{=}&\Exp_{\Prit_0}\left[\frac{\rmd\Prit_u(Z_{j-1})\cdot \rmd\Prit_s(Z_{j-1})}{\left(\rmd\Prit_0(Z_{j-1}) \right)^2} \cdot \I(A^{j} )G_{j}(v^{(1:j)})  \Exp\left[\frac{\rmd\Prit_u(Z_j|Z_{j-1})\cdot \rmd\Prit_s(Z_j|Z_{j-1})}{\left(\rmd\Prit_0(Z_j|Z_{j-1}) \right)^2}  \big{|} Z_{j-1} \right]\right]. 
		\end{eqnarray*}
		Where $(i)$ follows since the densities $ \frac{\rmd\Prit_u(Z_{j-1})\cdot \rmd\Prit_s(Z_{j-1})}{\left(\rmd\Prit_0(Z_{j-1}) \right)^2}$ are functions of $Z_{j-1}$, and $(ii)$ follows since $A^j$ and $G_j$ depend only on $\vone,\dots,\vj$, which is $Z_{j-1}$ measurable because $\{\vone,\wone,\dots,v^{(j-1)},w^{(j-1)}\}$ determine $\vj$. Finally, since Lemma~\ref{ConditionalLemma} implies that the conditional law $\Prit_{\theta}[\cdot|Z_{j-1}]$ is independent of $\wone,\dots,w^{(j-1)}$ given $\vone,\dots,v^{(j-1)}$, we have
		\begin{multline}
		\Exp\left[\frac{\rmd\Prit_u(Z_j|Z_{j-1})\cdot \rmd\Prit_s(Z_j|Z_{j-1})}{\left(\rmd\Prit_0(Z_j|Z_{j-1}) \right)^2}  \big{|} Z_{j-1} \right] \\
		= \Exp\left[\frac{\rmd\Prit_u(Z_j|Z_{j-1})\cdot \rmd\Prit_s(Z_j|Z_{j-1})}{\left(\rmd\Prit_0(Z_j|Z_{j-1}) \right)^2}  \big{|} \{\vone,\dots,\vj\}\right] := g_j(u,s;v^{(1:j)}).
		\end{multline}
		Moreover, on $A^j$, we have that $d\langle v^{(j)},\theta\rangle^2 \le \tau_j$ for $\theta \in \{u,s\}$, which implies
		\begin{eqnarray*}
		 \I(A^{j} )G_{j}(v^{(1:j)})g_j(u,s;v^{(1:j)}) &=& \I(A^{j} )g_j(u,s;v^{(1:j)}) \sup_{\vjplustil,\dots,\vktil \in \calV^k_{j+1}(Z_{j-1})} \prod_{i=j+1}^k g_i(u,s;\{v^{(1:j)},\widetilde{v}^{(j+1:i)}\})\\
		 &\le& \I(A^{j} ) \sup_{\widetilde{v}^{(j)},\dots,\vktil \in \calV^k_{j}(Z_{j-1})} \prod_{i=j}^k g_i(u,s;\{v^{(1:j-1)},\widetilde{v}^{(j+1:i)}\})\\
		 &=& \I(A^{j} ) G_{j-1}(v^{(1:j-1)}) \\
		 &\le&  \I(A^{j-1} ) G_{j-1}(v^{(1:j-1)}).
		\end{eqnarray*}
		where the last inequality follows since $A^j \subset A^{j-1}$. Altogether, we have proven
		\begin{eqnarray}
		\Exp_{\Prit_0}\left[G_{j}(v^{(1:j)}) \frac{\rmd\Prit_u(Z_j)\cdot \rmd\Prit_s(Z_j)}{\left(\rmd\Prit_0(Z_j) \right)^2} \I(A^{j} )\right] \le \Exp_{\Prit_0}\left[G_{j-1}(v^{(1:j-1)}) \frac{\rmd\Prit_u(Z_{j-1})\cdot \rmd\Prit_s(Z_{j-1})}{\left(\rmd\Prit_0(Z_{j-1}) \right)^2} \I(A^{j-1} )\right].
		\end{eqnarray}
		To conclude, we use the fact that $G_{k}(v^{(1:k)}) = 1$ and that $Z_0= \emptyset$ to render
		\begin{eqnarray*}
		 \Exp_{\Prit_0}\left[ \frac{\rmd\Prit_u(Z_k)\cdot \rmd\Prit_s(Z_k)}{\left(\rmd\Prit_0(Z_k) \right)^2} \I(A^{k} )\right] 
		&=& \Exp_{\Prit_0}\left[ G_{k}(v^{(1:k)}) \frac{\rmd\Prit_u(Z_k)\cdot \rmd\Prit_s(Z_k)}{\left(\rmd\Prit_0(Z_k) \right)^2} \I(A^{k} )\right]\\
		&\le& \Exp_{\Prit_0}\left[ G_{0}(v^{(1:0)}) \frac{\rmd\Prit_u(Z_0)\cdot \rmd\Prit_s(Z_0)}{\left(\rmd \Prit_0(Z_0) \right)^2} \I(A^{0} )\right] \\
		&=& \sup_{\vonetil,\dots,\vktil \in \calV^k_{j+1}(Z_{j-1})} \prod_{i=1}^k g_i(u,s;\widetilde{v}^{(1:k)}),
		\end{eqnarray*}
		as needed.

	\end{proof}

\subsection{Chi-Squared Computations}

We begin with the following Lemma:
\begin{lem}\label{Chi_Squared_Computation_Lemma} The following identity holds,
\begin{eqnarray}
\Exp_{\Prit_0}\left[\frac{\rmd\Prit_{u}(Z_i|Z_{i-1})\rmd\Prit_{s}(Z_i|Z_{i-1})}{(\rmd\Prit_0(Z_i|Z_{i-1}))^2}\right] = \exp\left\{\lambda^2 d \langle \vi,u\rangle\langle \vi,s \rangle\big(u^{\top}P_{i-1}\Sigma_i^{\dagger}P_{i-1}s\big)\right\}.
\end{eqnarray}
\end{lem}

\begin{proof}[Proof of Lemma~\ref{Chi_Squared_Computation_Lemma}]
	We have that
	\begin{eqnarray*}
	 \frac{\rmd\Prit_{u}(Z_i|Z_{i-1})}{\rmd\Prit_0(Z_i|Z_{i-1})}
	&=& \exp\left\{-\frac{1}{2}\left\|\lambda P_{i-1}uu^{\top}\vi-P_{i-1}\wi\right\|_{\Sigma_i/d}^2  +\frac{1}{2}\left\|P_{i-1}\wi\right\|^2_{\Sigma_i/d}\right\}\\
	&=& \exp\left\{-\frac{d\lambda^2}{2}\left\|P_{i-1}uu^{\top}\vi\right\|_{\Sigma_i}^2  + d\lambda u^{\top}\vi u^{\top}P_{i-1}\Sigma_{i}^{\dagger}P_{i-1}\wi\right\}.
	\end{eqnarray*}
	Thus,
	\begin{equation}\label{factored_exp_likelihood}
	\begin{aligned}
	 \Exp_{\Prit_0}\left[\frac{\rmd\Prit_{u}(Z_i|Z_{i-1})\rmd\Prit_{s}(Z_i|Z_{i-1})}{(\rmd\Prit_0(Z_i|Z_{i-1}))^2}\right]
	&= e^{-\frac{d^2\lambda}{2}\left(\left\|P_{i-1}uu^{\top}\vi\right\|_{\Sigma_i}^2 + \left\|P_{i-1}ss^{\top}\vi\right\|_{\Sigma_i}^2\right)} \\
	&~~~\times \Exp_{\Prit_0}\left[e^{d\lambda v^{(i)\top} (uu^\top + ss^\top) P_{i-1}\Sigma_{i}^{\dagger}P_{i-1}\wi}\right].
	\end{aligned}
	\end{equation}
	Now, note that if $Z \sim \calN(0,I_d)$ is a standard gaussian vector, then we have that $\wi \overset{d}{=} \frac{1}{\sqrt{d}}\Sigma_i^{1/2}Z$, under $\Prit_0$. Thus,
	\begin{equation}\label{exp_likelihood_2}
	\begin{aligned}
	& \Exp_{\Prit_0}\left[e^{d\lambda v^{(i)\top} (uu^\top + ss^\top) P_{i-1}\Sigma_{i}^{\dagger}P_{i-1}\wi}\right] \\
	&= \Exp_{Z} \left[e^{\sqrt{d}\lambda (\langle \vi, u \rangle u  + \langle \vi,s \rangle s)^{\top}P_{i-1}\Sigma_{i}^{\dagger/2}P_{i-1}Z}\right]\\
	 &= e^{\frac{d\lambda^2}{2} \left\|P_{i-1}\Sigma_{i}^{\dagger/2}P_{i-1}(\langle \vi, u \rangle u + \langle \vi,s \rangle s)\right\|^2}\\
	 &\overset{(i)}{=} e^{\frac{d\lambda^2}{2} \left\|P_{i-1}(\langle \vi, u \rangle u + \langle \vi,s \rangle s)\right\|_{\Sigma_i}^2}\\
	  &= e^{\frac{d\lambda^2}{2} \left(\left\|\langle \vi, u \rangle P_{i-1} u\right\|^2_{\Sigma_i} + \left\|\langle \vi, s \rangle P_{i-1}s\right\|^2_{\Sigma_i}\right) + d\lambda^2 \langle \vi, u \rangle \langle \vi , s \rangle u^{\top}P_{i-1}\Sigma_i^{\dagger}P_{i-1}s },
	\end{aligned}
	\end{equation}
	where $(i)$ uses the fact that $P_{i-1}$ is symmetric, idempotent, and commutes with $\Sigma_i$.
	Combining Equations~\eqref{factored_exp_likelihood} and~\eqref{exp_likelihood_2}, we have
	\begin{eqnarray*}
	 \Exp_{\Prit_0}\left[\frac{\rmd\Prit_{u}(Z_i|Z_{i-1})\rmd\Prit_{s}(Z_i|Z_{i-1})}{(\rmd\Prit_0(Z_i|Z_{i-1}))^2}\right]
	&=& e^{d\lambda^2 \langle \vi, u \rangle \langle \vi , s \rangle u^{\top}P_{i-1}\Sigma_i^{\dagger}P_{i-1}s}.
	\end{eqnarray*}
\end{proof}

\begin{prop}\label{Inner_product_constraints_cor} Fix $\tau_1,\dots,\tau_T$, and let $A^T_u$ denote the event that $d\langle v_i,u\rangle^2 \le \tau_i$ for all $i \in [T]$. Then, 
\begin{eqnarray}
 \Exp_{\Prit_0}\left[\frac{\rmd\Prit_u(Z_T ; A^T_{u_1})\rmd\Prit_s(Z_T;A^T_{u_2})}{\left(\rmd\Prit_0(Z_T )\right)^2}  \right] \le e^{\lambda^2\{|\langle u_1 , u_2 \rangle| \sum_{i=1}^T \tau_i + \frac{1}{d}(\sum_{i=1}^T \tau_i)^2\}}
\end{eqnarray}
\end{prop}
\begin{proof}[Proof of Proposition~\ref{Inner_product_constraints_cor}] We have that
\begin{eqnarray*}
(P_{i-1}\Sigma_i P_{i-1})^{\dagger} &=& \left(P_{i-1}(I + v^{(i)} v^{(i)\top}) P_{i-1}\right)^{\dagger} \\
&=& P_{i-1}(I + v^{(i)} v^{(i)\top})^{-1} P_{i-1} \\
&=& P_{i-1}(I - \frac{1}{2} v^{(i)}v^{(i)\top}) P_{i-1} \\
&=& P_{i-1}- \frac{1}{2} v^{(i)}v^{(i)\top} \\
&=& I - \sum_{j=1}^{i-1}v^{(j)}v^{(j)\top}- \frac{1}{2}v^{(i)}v^{(i)\top},
\end{eqnarray*}
where we use the fact that $P_{i-1},v^{(i)}v^{(i)\top}$ and $I$ commute throughout, and orthonormality of $\vone,\dots,\vi$. Under $A^k_u\cap A^k_s$, we have that $|\langle \vi, \theta \rangle| \le \sqrt{\tau_i/d}$ for $\theta \in \{u,s\}$, which implies that
\begin{eqnarray*}
&& (v^{(i)\top} u)(v^{(i)\top} s)\{u^{\top}P_{i-1}\Sigma_i^{\dagger}P_{i-1}s\} \\
&=& (v^{(i)\top} u)(v^{(i)\top} s)u^{\top}s - \frac{1}{2} (v^{(i)\top} u)^2(v^{(i)\top} s)^2 - (v^{(i)\top} u)(v^{(i)\top} s)\sum_{j=1}^{i-1}(v^{(j)\top} s)(v^{(j)\top} u)\\
&\le& \frac{\tau_i}{d} |u^{\top}s| + \frac{1}{d^2}\sum_{j=1}^{i-1}\tau_i\tau_j.
\end{eqnarray*}
Thus, by Lemma~\ref{Chi_Squared_Computation_Lemma}, we have for any $(\vone,\dots,\vk) \in \calV^k_u \cap \calV^k_s$ that
\begin{eqnarray}
g_i(u,s,0;\{\vj\}_{1\le j \le i}) \le e^{\lambda^2\{\tau_i |u^{\top}s| + d^{-1}\sum_{j=1}^{i-1}\tau_i\tau_j\}}.
\end{eqnarray}
Equation~\eqref{product_eq} completes the demonstration, since for any $(\vone,\dots,\vk) \in \calV^k_u \cap \calV^k_s$,
\begin{eqnarray}
\prod_{i=1}^k g_i(u,s,0;\{\vj\}_{1\le j \le i}) \le e^{\lambda^2\sum_{i=1}^k\{\tau_i |u^{\top}s| + d^{-1}\sum_{j=1}^{i-1}\tau_i\tau_j\}}
\le e^{\lambda^2\left(\sum_{i=1}^k \tau_i |u^{\top} s| + d^{-1}(\sum_{i=1}^k \tau_i)^2\right)}.
\end{eqnarray}
\end{proof}

\section{Additional Proofs}

\subsection{Proof of Theorem~\ref{Main_estimation_theorem} \label{sec:Main_estimation_thm_proof}}
For $k = 1$, we immediately have
\begin{eqnarray}\label{k_eq_1_case}
\Exp_{\theta\sim\calP}\Pr_{\theta}\left[d\langle \vone,\theta \rangle^2 > \tau_{1}\right] \le \sup_{v \in \calS^{d-1}}\Exp_{\theta\sim\calP}\left[d\langle v,\theta \rangle^2 > \tau_{1}\right]. 
\end{eqnarray}
For $k \in \{2,\dots,T+1\}$, we define the action space $\calA = \calS^{d-1}$, and loss function $\calL(\vkplus,\theta) = \I(\langle \vkplus,\theta \rangle \ge \tau_{k+1})$. Applying Proposition~\ref{chi_sq_fano_prop}, then the inequality~\eqref{Chi_Plus_1_Eq}, we have
\begin{multline}\label{chi_sq_main_recursion}
\Exp_{\theta\sim\calP}\Pr_{\theta}\left[d\langle \vkplus,\theta \rangle^2 > \tau_{k}; A_{\theta}^{k}\right] \\ 
\le
\sup_{u \in \calS^{d-1}}\Pr_{\theta \sim \calP}\left[d\langle u,\theta \rangle^2 \ge \tau_{k}\right] + \sqrt{\Exp_{\theta \sim \calP}\Exp_{\Prit_0}\left[\left(\frac{\mathrm{d}\Prit_\theta(Z_{k} ; A_\theta^{k})}{\mathrm{d}\Prit_0(Z_{k})}\right)^2 \right]\sup_{v \in \calS^{d-1}}\Pr_{\theta \sim \calP}\left[d\langle v,\theta \rangle^2 \ge \tau_{k}\right] }\\
\le
\sup_{u \in \calS^{d-1}}\Pr_{\theta \sim \calP}\left[d\langle u,\theta \rangle^2 \ge \tau_{k}\right] + 
e^{\frac{\lambda^2}{2}\sum_{i=1}^{k} \tau_i}\sqrt{\sup_{v \in \calS^{d-1}}\Pr_{\theta \sim \calP}\left[d\langle v,\theta \rangle^2 \ge \tau_{k}\right] }\\
\le 2
e^{\frac{\lambda^2}{2}\sum_{i=1}^k \tau_i}\sqrt{\sup_{v \in \calS^{d-1}}\Pr_{\theta \sim \calP}\left[d\langle v,\theta \rangle^2 \ge \tau_{k}\right] }.
\end{multline}
Hence,
\begin{align*}
\Exp_{\theta \sim \calP}\Pr_{\theta}&\left[\exists k \in [T+1]: \langle \vk , \theta \rangle^2 > \frac{\tau_k}{d}\right]\\ 
&= \sum_{k=1}^{T+1} \Exp_{\theta \sim \calP}\Pr_{\theta}\left[\left\{\langle \vk , \theta \rangle^2 > \frac{\tau_k}{d} \right\}\cap \left\{\forall j < k, \langle \vj , \theta \rangle^2 \le \frac{\tau_j}{d}\right\}\right ]\\
&= \sum_{k=1}^{T+1} \Exp_{\theta \sim \calP}\Pr_{\theta}\left[\left\{\langle \vk , \theta \rangle^2 > \frac{\tau_k}{d} \right\};A_{\theta}^{k-1}\right]\\
&\overset{(i)}{\le} \sup_{v \in \calS^{d-1}}\Pr_{\theta \sim \calP}[d\langle v, \theta \rangle^2 \ge \tau_1] + \sum_{k=1}^{T} \Exp_{\theta \sim \calP}\Pr_{\theta}\left[\left\{\langle \vkplus , \theta \rangle^2 > \frac{\tau_{k+1}}{d} \right\};A_{\theta}^{k-1}\right]\\
&\overset{(ii)}{\le} \sup_{v \in \calS^{d-1}}\Pr_{\theta \sim \calP}[d\langle v, \theta \rangle^2 \ge \tau_1] + 2\sum_{k=1}^{T}
e^{\frac{\lambda^2}{2}\sum_{i=1}^k \tau_i}\sqrt{\sup_{v \in \calS^{d-1}}\Pr_{\theta \sim \calP}\left[d\langle v,\theta \rangle^2 \ge \tau_{k+1}\right] }.
\end{align*}
where $(i)$ uses Equation~\eqref{k_eq_1_case}, and $(ii)$ uses~\eqref{chi_sq_main_recursion}. 
The theorem follows from a union bound by summing up the above display and combining with Equation~\eqref{k_eq_1_case}.


\subsection{Proof of Corollary~\ref{Main_Cor_estimation}\label{Main_Cor_Proof}}

Set $\tau_1 = 2(\sqrt{\log(1/\delta)} + 1)^2$, so that
\begin{eqnarray*}
\frac{1}{2}(\sqrt{\tau_1} - \sqrt{2})^2 = \log(1/\delta).
\end{eqnarray*}
Now, define $\tau_k \ge \sqrt{2}$ via
\begin{eqnarray*}
\frac{1}{2}(\sqrt{\tau_k} - \sqrt{2})^2 = \lambda^2\sum_{i=1}^{k-1}\tau_i + (k-1)\tau_1.
\end{eqnarray*}
Then, using the fact that $\Pr_{\theta \sim \calP}[d\langle v, \theta \rangle^2 \ge \tau] \le \exp(\frac{-1}{2}(\sqrt{\tau} - \sqrt{2})^2)$ from Lemma~\ref{SphereConcentation}, Theorem~\ref{Main_estimation_theorem} implies that
\begin{eqnarray*}
&& \Exp_{\theta \sim \calP}\Pr_{\theta}\left[\exists k \in [T+1]: \langle \vk , \theta \rangle^2 > \frac{\tau_k}{d}\right] \\
&\le&  \exp(\frac{-1}{2}(\sqrt{\tau_1} - \sqrt{2})) + \sum_{k=2}^{T+1} \sqrt{\exp(\lambda^2 \sum_{i=1}^{k-1}\tau_i - \frac{1}{2}(\sqrt{\tau_k} - \sqrt{2})^2)}\\
&=&  \exp(\frac{-1}{2}(\sqrt{\tau_1} - \sqrt{2})) + \sum_{k=2}^{T+1} \exp(- \frac{k-1}{2}\tau_1)\\
&\le&  \exp(\frac{-1}{2}(\sqrt{\tau_1} - \sqrt{2})) + \sum_{k=2}^{T+1} \exp(- \frac{k-1}{2}(\tau_1-\sqrt{2})^2)\\
&=&  \delta + \sum_{k=2}^{T+1} \delta^{k-1} \le \frac{2\delta}{1-\delta}.
\end{eqnarray*}
We now bound that rate at which our $\tau_k$ increase. Since $\tau_i$ are non-decreasing, we have
\begin{eqnarray*}
\frac{1}{2}(\sqrt{\tau_k} - \sqrt{2})^2 = \frac{\tau_k}{2}(1 - \sqrt{2/\tau_k}) \ge \frac{\tau_k}{2}(1 - \sqrt{2/\tau_1}).
\end{eqnarray*}
And thus,
\begin{eqnarray*}
\tau_k \le \frac{2}{1 - \sqrt{2/\tau_1}}\cdot(\lambda^2\sum_{i=1}^{k-1}\tau_i + (k-1)\tau_1) \le \frac{2(\lambda^2 + 1)}{1-\sqrt{2/\tau_1}}\sum_{i=1}^{k-1}\tau_i.
\end{eqnarray*}
For ease, set $\alpha = \frac{2(\lambda^2 + 1)}{1-\sqrt{2/\tau_1}}$, and consider the comparison sequence $\tau_1' = \tau_1$, and $\tau_k' = \alpha\sum_{i=1}^{k-1}\tau_i'$. Then $\tau_k' \ge \tau_k$, and moreover, $\tau_k' \ge \alpha \tau_{k-1}' $, which implies that
\begin{eqnarray*}
\sum_{i=1}^{k-1}\tau_i' = \tau_{k-1}'\sum_{i=1}^{k-1}\frac{\tau_i'}{\tau_{k-1}'} \le \tau_{k-1}\sum_{i=0}^{k-2}\alpha^{-i} \le \frac{\tau_{k-1}'}{1-1/\alpha}.
\end{eqnarray*}
Thus, $\tau_k \le \tau_{k}' = \alpha\sum_{i=1}^{k-1}\tau_i' \le \frac{\alpha}{1 - 1/\alpha}\tau_{k-1}' \le \tau_1'\left(\frac{\alpha}{1 - 1/\alpha} \right)^{k-1} = \tau_1 \left(\frac{\alpha}{1 - 1/\alpha} \right)^{k-1}$. Finally, we bound
\begin{eqnarray*}
&& \frac{\alpha}{1 - 1/\alpha} = 2\lambda^2 \cdot \frac{1 + 1/\lambda^2}{(1 - \frac{1 - \sqrt{2/\tau_1}}{2(\lambda^2 + 1)})(1 - \sqrt{2/\tau_1})} \\
&\le& 2\lambda^2 \cdot \frac{1 + 1/\lambda^2}{(1 - 1/2\lambda^2)(1 - \sqrt{1/(1+\sqrt{\log(1/\delta)} )^2})}\\
&=& 2\lambda^2 \cdot \frac{1 + 1/\lambda^2}{(1 - 1/2\lambda^2)(1 - \sqrt{1/(1+ \log(1/\delta)})}\\
&=& 2\lambda^2 \cdot c(\lambda,\delta).
\end{eqnarray*}
which implies that
\begin{eqnarray*}
\tau_k \le \tau_1 \left(2\lambda^2 \cdot c(\lambda,\delta)\right)^{k-1} = 2(\sqrt{\log(1/\delta)} + 1)^2\left(2\lambda^2 \cdot c(\lambda,\delta)\right)^{k-1}.
\end{eqnarray*}

	\subsection{Proof of Proposition~\ref{estimation_cor_2}~\label{Est_Cor_proof}}
	Fix any $\lambda \ge 1$. Recall the function  $c(\lambda,\delta) := (1 + 1/\lambda^2)\left\{(1 - 1/2\lambda^2)(1 - 1/\sqrt{1+\log(1/\delta)}\right\}^{-1}$ from Corollary~\ref{Main_Cor_estimation}. For $\delta \le 1/e$ and $\lambda \ge 1$, we have
	 \begin{eqnarray*}
	 c(\lambda,\delta) \le (1 + 1/\lambda^2)\left\{(1 - 1/2\lambda^2)(1 - \sqrt{1+\log(1/\delta)}\right\}^{-1} \le \frac{2}{(1/2)(1 - 1/\sqrt{2})} := c'.
	 \end{eqnarray*}
	Next, fix any $\eta \ge 0$ and choose $\delta$ such that
	\begin{eqnarray}\label{delta_cor_eq}
	\log (1/\delta) \le \left(\sqrt{\frac{d \eta}{2 \left(2\lambda^2 \cdot c'\right)^{T}} } - 1 \right)^2.
	\end{eqnarray}
	If we suppose that
	\begin{eqnarray}\label{eta_eq}
	\sqrt{\frac{d \eta}{2 \left(2\lambda^2 \cdot c'\right)^{T}} } \ge 2.
	\end{eqnarray}
	This implies that $\delta \le 1/e$, so $c'\ge c(\lambda,\delta)$ and thus 
	\begin{eqnarray}
	\eta \ge (\left(2\lambda^2 \cdot c'\right)^{T}\cdot \frac{2\left(\sqrt{\log(1/\delta)} + 1\right)^2}{d} \ge (\left(2\lambda^2 \cdot c(\lambda,\delta)\right)^{T}\cdot \frac{2\left(\sqrt{\log(1/\delta)} + 1\right)^2}{d}.
	\end{eqnarray}
	Then
	\begin{eqnarray*}
	&&\Exp_{\theta \sim \calP}\Pr_{\theta}\left[ \langle \widehat{v}, \theta \rangle^{2} \ge \eta \right] \\
	&\le& \Exp_{\theta \sim \calP}\Pr_{\theta}\left[ \langle \widehat{v}, \theta \rangle^{2} \ge (\left(2\lambda^2 \cdot c(\lambda,\delta)\right)^{T}\cdot \frac{2\left(\sqrt{\log(1/\delta)} + 1\right)^2}{d} \right] \\
	&\le& \Exp_{\theta \sim \calP}\Pr_{\theta}\left[ \langle \vk, \theta \rangle^{2} \ge (\left(2\lambda^2 \cdot c(\lambda,\delta)\right)^{k-1}\cdot \frac{2\left(\sqrt{\log(1/\delta)} + 1\right)^2}{d} \forall k \in [T+1]\right] \\
	&\overset{(i)}{\le}&\frac{2\delta}{1-\delta} \overset{(ii)}{=} \frac{2\delta}{1-1/e},
	\end{eqnarray*}
	where $(i)$ uses Corollary~\ref{Main_Cor_estimation} and $(ii)$ uses that $\delta \le 1/e$. Moreover, by Equation~\eqref{delta_cor_eq}, we have
	\begin{eqnarray*}
	\delta = \exp(-\log(1/\delta)) &\le& \exp\left(- \left(\sqrt{\frac{d \eta}{2 \left(2\lambda^2 \cdot c'\right)^{T}} } - 1 \right)^2\right)\\
	&\le& \exp\left(- \left(\frac{1}{2}\sqrt{\frac{d \eta}{2 \left(2\lambda^2 \cdot c'\right)^{T}} }\right)^2\right)\\
	&=& \exp\left(- \frac{d \eta}{8 \left(2\lambda^2 \cdot c'\right)^{T}} \right).
	\end{eqnarray*}
	Thus, if $\sqrt{\frac{d \eta}{2 \left(2\lambda^2 \cdot c'\right)^{T}} } \ge 2$, we have
	\begin{eqnarray*}
	\Exp_{\theta \sim \calP}\Pr_{\theta}\left[ \langle \widehat{v}, \theta \rangle^{2} \ge \eta \right] \le \frac{2}{1-e} \cdot \exp\left(- \frac{d \eta}{8 \left(2\lambda^2 \cdot c'\right)^{T}} \right).
	\end{eqnarray*}
	On the other hand, if $\sqrt{\frac{d \eta}{2 \left(2\lambda^2 \cdot c'\right)^{T}} } < 1$, then the right hand side of the above display is at least $1$, so the result also holds vacuously.

\subsection{Proof of Lemma~\ref{Detection_Prop}\label{DetecLemProof}}
Let $c(\lambda,\delta)$ be as above, and fix $\delta \le 1/2$ and $\lambda \ge 1$. Then, 
\begin{eqnarray*}
\Exp_{\Prit_0}\left[\left(\frac{\rmd\overline{\Prit}[\cdot;\{A_{\theta}^T\}]}{\rmd\Prit_0}\right)^{2}\right] &=& \Exp_{\theta,\theta' \sim \calP}\Exp_{\Qit}\left[\frac{\rmd\Prit_{\theta}[\cdot;A_{\theta}^T]\rmd\Prit_{\theta'}[\cdot;A_{\theta'}^T]}{(\rmd\Qit)^2}\right] \\
&=& \Exp_{\theta,\theta' \sim \calP}e^{\lambda^2\{|\langle \theta , \theta' \rangle| \sum_{i=1}^T \tau_i + \frac{1}{d}(\sum_{i=1}^T \tau_i)^2\}}\\
&=& e^{\frac{\lambda^2}{d}(\sum_{i=1}^T \tau_i)^2} \cdot \Exp_{\theta,\theta' \sim \calP}e^{\lambda^2(\sum_{i=1}^T \tau_i )\{|\langle \theta , \theta' \rangle| \}}\\
&\le& e^{\frac{\lambda^2}{d}(\sum_{i=1}^T \tau_i)^2} \cdot e^{\frac{4\lambda^4}{d}(\sum_{i=1}^T \tau_i )^2 + \lambda^2(\sum_{i=1}^T \tau_i )\sqrt{2/d}}\\
&=& \exp \left\{\left(\frac{2\lambda^2\sqrt{1 + 1/4\lambda^2} }{\sqrt{d}}\left(\sum_{i=1}^T \tau_i \right)\right)^2+ \frac{\sqrt{2}\lambda^2}{\sqrt{d}}\left(\sum_{i=1}^T \tau_i \right)\right\}\\
&\overset{(i)}{\le}& \exp \left\{\frac{(2\sqrt{5/4} + \sqrt{2})\lambda^2}{\sqrt{d}}\cdot \sum_{i=1}^T \tau_i\right\} \le  \exp \left\{\frac{\sqrt{7}\lambda^2}{\sqrt{d}}\cdot \sum_{i=1}^T \tau_i\right\},
\end{eqnarray*}
where $(i)$ holds as long as $\frac{(2\sqrt{5/4} + \sqrt{2})\lambda^2}{\sqrt{d}}\cdot \sum_{i=1}^T \tau_i \le 1$ and $\lambda \ge 1$. On the other hand, from Corollary~\ref{Main_Cor_estimation}, we can choose $\tau_1,\dots, \tau_T$ such that, $\Exp_{\theta}\Prit_{\theta}[A_{\theta}^T] \ge \frac{2\delta}{1-\delta}$, and 
\begin{eqnarray*}
\sum_{i=1}^T\tau_i \le 4(2\lambda^2 \cdot c(\lambda,\delta))^{T-1} \left(\sqrt{\log(1/\delta)} + 1\right)^2.
\end{eqnarray*}
Then, as long as $4(2\lambda^2 \cdot c(\lambda,\delta))^{T-1} \left(\sqrt{\log(1/\delta)} + 1\right)^2 \le 1$, we have 
\begin{eqnarray*}
\Exp_{\Prit_0}\left[\left(\frac{\rmd\overline{\Prit}[\cdot;\{A_{\theta}^T\}]}{\rmd\Prit_0}\right)^{2}\right] &\le& \exp\left\{\frac{4\sqrt{7}\lambda^2}{\sqrt{d}}(2\lambda^2 \cdot c(\lambda,\delta))^{T-1} \left(\sqrt{\log(1/\delta)} + 1\right)^2\right\}\\
&\overset{(i)}{\le}& \exp\left\{\frac{4\sqrt{7}}{\sqrt{d}}(2\lambda^2 \cdot c(\lambda,\delta))^{T} \left(\sqrt{\log(1/\delta)} + 1\right)^2\right\}.
\end{eqnarray*}
where $(i)$ uses that $c(\lambda,\delta) \ge 1$. Since $\delta \le 1/2$ and $\lambda \ge 1$, we have that
\begin{eqnarray*}
c(\lambda,\delta) \le c'' := \max_{\delta \le 1/2,\lambda \ge 1}c(\lambda,\delta) = c(1,1/2) < \infty.
\end{eqnarray*}
We can then bound
\begin{eqnarray*}
 \Exp_{\Qit}\left[\left(\frac{\rmd\overline{\Prit}[\cdot;\{A_{\theta}^T\}]}{\rmd\Qit}\right)^{2}\right] \le \exp\left\{\frac{4\sqrt{7}\left(c'' \lambda^2\right)^T}{\sqrt{d}} \cdot \left(\sqrt{\log(1/\delta)} + 1\right)^2\right\}.
\end{eqnarray*}
First, suppose that the quantity in the above exponential is less than $1/2$, then using the inequality $e^{x} - 1 \le 2x$ for $x \le 1/2$, we can bound 
\begin{eqnarray}\label{exminusbound}
\Exp_{\Prit_0}\left[\left(\frac{\rmd\overline{\Prit}[\cdot;\{A_{\theta}^T\}]}{\rmd\Prit_0}\right)^{2}\right] \le \exp\left\{\frac{4\left(c'' \lambda^2\right)^T}{\sqrt{d}} \cdot \left(\sqrt{\log(1/\delta)} + 1\right)^2\right\} - 1 \le \frac{8\left(c'' \lambda^2\right)^T}{\sqrt{d}} \cdot \left(\sqrt{\log(1/\delta)} + 1\right)^2. 
\end{eqnarray}
Take $\delta = \frac{\left(c'' \lambda^2\right)^T}{\sqrt{d}}$. If $\delta \le 1/2$ and $\frac{8\left(c'' \lambda^2\right)^T}{\sqrt{d}} \cdot \left(\sqrt{\log(1/\delta)} + 1\right)^2  \le 1$, then
\begin{eqnarray*}
\|\Prit_0 - \overline{\Prit}\|_{TV} &\overset{(i)}{\le}& \frac{1}{2}\sqrt{\Exp_{\Qit}\left[\left(\frac{\rmd\overline{\Prit}[\cdot;\{A_{\theta}^T\}]}{\rmd\Qit}\right)^{2}\right] - 1} + \frac{\sqrt{2\Exp_{\theta\sim \calP}\Prit_{\theta}[A_{\theta}^T]} +\Exp_{\theta\sim \calP}\Prit_\theta[A_{\theta}^T]}{2}\\
&\overset{(ii)}{\le}&\left(\sqrt{\log(1/\delta)} + 1\right)\cdot\sqrt{\frac{2\left(\gamma'' \lambda^2\right)^T}{\sqrt{d}}  } +  \frac{\sqrt{2\Exp_{\theta \sim \calP}\Prit_\theta [A_{\theta}^T]} +\Exp_{\theta \sim \calP}\Prit_\theta[A_{\theta}^T]}{2}\\
&\overset{(iii)}{\le}&\left(\sqrt{\log(1/\delta)} + 1\right)\cdot\sqrt{\frac{2\left(\gamma'' \lambda^2\right)^T}{\sqrt{d}}  } + \frac{\sqrt{2(\frac{2\delta}{1-\delta})}+ \frac{2\delta}{1-\delta}}{2}\\
&\le& \left(\sqrt{\log(1/\delta)} + 1\right)\cdot\sqrt{\frac{2\left(\gamma'' \lambda^2\right)^T}{\sqrt{d}}  } + 4\sqrt{2\delta} \\
&\le& \sqrt{2}(\log^{1/2}\frac{\sqrt{d}}{\left(\gamma'' \lambda^2\right)^T} + 4)\cdot \sqrt{\frac{\left(\gamma'' \lambda^2\right)^T}{\sqrt{d}}}.
\end{eqnarray*}
where $(i)$ uses Propostion~\ref{truncChi_detec}, $(ii)$ uses Equation~\eqref{exminusbound}, and $(iii)$ uses the fact that $\Exp_{\theta\sim \calP}\Prit_{\theta}[A_{\theta}^T] \le 2\delta/(1-\delta)$. On the other hand, if $\delta \ge 1/2$, or if  $\frac{8\left(c'' \lambda^2\right)^T}{\sqrt{d}} \cdot \left(\sqrt{\log(1/\delta)} + 1\right)^2  \ge 1$, then the last line of the above display is at least $1$, so the conclusion remained true because $\|\Prit_0 - \overline{\Prit}\|_{\TV}\le 1$.

\section{Information Theoretic Tools\label{Info_Tools}}
The purpose of this section is to extend the theory of $f$-divergence based lower bounds (\cite{csiszar1972class,guntuboyina2011lower,liese2012phi}) to handle non-normalized measures. In particular, this will allow us to prove analogues of the Fano-style Bayes Risk Lower Bounds in~\cite{chen2016bayes} which hold for truncated probability distributions, whose mass does not add up to one. The motivation for considering truncated measures is that we can restrict to parts of the probability space where the likelihood rations between alternatives do not ``blow-up'', allowing us to prove stronger bounds on the amount of information gained. 
\subsection{Preliminaries}
Let $\mu$ be a finite, non-negative measure on a space $(\calX,\calF)$, and denote $|\mu| = \mu(\calX)$. We say that $\mu$ is positive if $|\mu| > 0$. Abusive notation, we define $\Exp_{X\sim \mu}[f(X)]$ to denote integration of $f(X)$ with respect to the measure $\mu$. Throughout, we will also let $f: (0,\infty) \to \R$ denote a convex function define $f'(\infty):= \lim_{x \to \infty} f(x)/x$. We allow the latter limit to be infinite.

\begin{defn} For a finite, non-negative measure $\mu$ and finite positive measure $\nu$ over the class $(\calX,\calF)$, and a convex $f:(0,\infty) \to \R$, we define the (generalized) $f$-divergence between $\mu$ and $\nu$ as
\begin{eqnarray}
D_{f}(\mu,\nu) := \int_{x \in \calX: \rmd\nu(x) > 0} f\left(\frac{\rmd\mu}{\rmd\nu}\right)d\nu + \mu\left(\{\rmd\nu = 0\}\right)\cdot f'(\infty)
\end{eqnarray}
with the convention $0 \cdot f'(\infty) = 0$. 
\end{defn}
\begin{rem}\label{measure_theory_remak}
For a formal definition of how to interpret the notation $\rmd\mu,\rmd\nu$, see for example~\cite{kallenberg2006foundations}. Recall that, for measures $\mu$ and $\nu$ on $(\calX,\calF)$, we say $\mu \ll \nu$, or $\mu$ is absolutely continuous with respect to $\nu$, if, for any set $A \in \calF$, $\nu(A) = 0$ implies $\mu(A) = 0$. In this case, the Radon-Nikodym derivative $\rmd\mu/\rmd\nu$ is well defined~\cite{kallenberg2006foundations}, and $D_{f}(\mu,\nu) := \int_{x \in \calX: \rmd\nu(x) > 0} f\left(\frac{\rmd\mu}{\rmd\nu}\right)\rmd\nu$.
\end{rem}

The $f$-divergences play a substantial role in modern information theory, dating back to the work of Csisz\'ar~\cite{csiszar1972class}. They also generalize many classical information divergence; for example, taking the function $f(x) = x^2 - 1$ yields the classical $\chi^2$ divergence, and $f(x) = x \log x$ corresponds to the $\KL$-divergence. Recall that $\chi^2+1$ divergence from Definition~\ref{chi_plus_one_definition} corresponds to the case where $f(x) = x^2$, motivating the label ``$\chi^2+1$-divergence''. As defined in Definition~\ref{chi_plus_one_definition}, $D_{\chi^2+1}$ is not an $f$ divergence in the traditional sense, since commonly one requires that $f(1) = 0$, and that $D_f(\cdot,\cdot)$ takes probability distributions as its arguments. In particular, this ensures that $0 = D_f(\mu,\mu) \le \inf_{\nu} D_f(\mu,\nu)$~\cite{liese2012phi}. Luckily, the following lemma establishes that many of the nice properties of $f$ divergences carry through when these restrictions are lifted. 
\begin{lem}\label{fDivProperties} Let $\mu,\nu$ be two finite positive measures on a space $(\calX,\calF)$. Then
\begin{enumerate} 
\item \textbf{Convexity:} $D_{f}(\mu,\nu)$ is jointly convex in $\mu$ and $\nu$ over the convex set $(\mu,\nu): \mu \ll \nu$
\item \textbf{Distance-Like:} Suppose that $f'(\infty) \ge 0$. Then $D_{f}(\mu,\nu) \ge |\nu|f(|\mu|/|\nu|)$, which is attained when $\rmd\mu/ \rmd\nu = |\mu|/|\nu|$.
\item \textbf{Normalization:} Define $f(x;p,q) = qf(\frac{p}{q}x)$. Then,
\begin{eqnarray}
D_{f}(\mu,\nu) = D_{f(x;|\mu|,|\nu|)}(\mu/|\mu|,\nu/|\nu|).
\end{eqnarray}
\item \textbf{Linearity} $D_{\beta f+\alpha}(\mu,\nu) = \alpha|\nu| + \beta D_{f}(\mu,\nu)$
\item \textbf{Data-Processing:} Let $\Gamma$ be a measurable map from $(\calX,\calF) $ to $(\mathcal{Y},\calG)$, and let $\mu\Gamma^{-1}$ denote the pullback measure on $\mathcal{Y}$ given by $\mu\Gamma^{-1}(B) = \mu(\Gamma^{-1}(B))$ for all $B \in \calG$. Then,
\begin{eqnarray*} 
D_f(\mu,\nu) \ge D_f(\mu\Gamma^{-1},\nu\Gamma^{-1})~.
\end{eqnarray*} 
\end{enumerate}
\end{lem}
It is often useful to consider $f$ divergences on binary spaces (e.g. $|\calX| = 2$). In this regime, lemma~\ref{fDivProperties} immediately implies the following corollary
\begin{cor}\label{phi_cor}
For $a \in [0,p]$, $b \in [0,q]$, let $\Qit_{a,p}$ (resp.\ $\Qit_{b,q}$) denote the measures on $\{0,1\}$ which place mass $a$ (resp.\ $b$) on $1$, and $p-a$ (resp.\ $q-b$) on $\{0,1\}$, and define
\begin{eqnarray}\label{gen_phi_f_def}
\phi_f(a,b;p,q) = D_f(\Qit_{a,p},\Qit_{b,q}). 
\end{eqnarray}
Then, For $a\in[0,q],b \in [0,q]$, 
\begin{eqnarray}
	\phi_f(a,b;p,q) = bf(\frac{a}{b}) + (q-b)\cdot f(\frac{p-a}{q-b}), 
\end{eqnarray}
where the case $b = 0$ or $b = q$ is understood by taking the limits $b \to 0^+$ and $b \to q^-$. Moreover,

\begin{enumerate}
	\item Let $\mu,\nu$ be measures on a space $(\calX,\calF)$. For any event $A \in \calF$, we have 
	\begin{eqnarray}
	D_{f}(\mu,\nu) \ge \phi_f(\mu(A),\nu(A);|\mu|,|\nu|).
	\end{eqnarray} 
	\item The funciton $\phi_f(a,b;p,q)$ is jointly convex (and finite) in $a$ and $b$ for $(a,b) \in [0,p] \times (0,q)$, and 
	\begin{eqnarray}
	\lim_{b \to 0} \phi_f(a,b;p,q) = \phi(a,0;p,q) & \lim_{b \to q} \phi_f(a,b;p,q) = \phi(a,q;p,q).
	\end{eqnarray} 
	\item As a function of $a$ for fixed $b \in [0,q]$,  $\phi_f(a,b;p,q)$ is minimized when $a = (p/q)b$, and is therefore nondecreasing for $(q/p)a \ge b$. As a function of $b$ for fixed $a \in [0,p]$, $\phi_f(a,b;p,q)$ is minimized when $b = a(q/p)$ and is therefore nonincreasing for $(p/q)b \le a$.  
\end{enumerate}
\end{cor}
\begin{proof}  The first point follows from applying the Data Processing Inequality from Part 5 of Lemma~\ref{fDivProperties} to the RHS of Equation~\eqref{gen_phi_f_def}. The second point is an analogue of Part 1 of Lemma~\ref{fDivProperties}, but can be seen directly by noting that, if $f$ is a convex function, the perspective map $(a,b) \mapsto bf(a/b)$ is convex (see \cite{boyd2004convex}), and the limits follow from direct computation and the definition of $f'(\infty)$. The third point of Lemma~\ref{fDivProperties} to the RHS of Equation~\eqref{gen_phi_f_def}, and noting that $1$-d convex functions are non-increasing to the left (resp.\ non-decreasing to the right) of their minimizers. 
\end{proof}
\subsection{A Generalized Bayes-Risk Lower Bound\label{Gen_Bayes_risk}}
	With this in hand, we can prove generalization of the Bayes risk lower bounds from \cite{chen2016bayes}. We consider a space of actions $\calA$ (think queries $\vi$), and a $\{0,1\}$-loss function $\calL: \calA \times \Theta \to \{0,1\}$. As in~\cite{chen2016bayes}, we consider an decision rule $\fraka: \calX \to \calA$, and study $\calL(\fraka(X),\theta)$, where $X$ is drawn from a measure $\mu_{\theta}$ and $\theta$ comes from a prior $\calP$ over a space $(\Omega,\calG)$. Rather than lower bounding the probability that $\calL(\fraka(X),\theta)$ is equal to one $1$ (i.e., risk), we \emph{upper bound} the probability that $\calL(\frak(X),\theta)$ is equal $0$, which we call the \emph{value}. This will be easier to work with for our purposes, and makes more sense semantically when $|\mu_{\theta}| \le 1$.

	\begin{thm}[Generalized Bayes Risk Lower Bound]\label{GenFano} Let $\calL: \calA \times \Theta \to \{0,1\}$, let $\fraka$ denote a decision rule from $\calX \to \calA$, let $\calP$ be a probability distribution over $(\Theta,\calG)$, let $\nu$ and $\{\mu_{\theta}\}$ be a family of finite measures over $(\calX,\calF)$ such that $\Exp_{\theta \sim \calP}|\mu_{\theta}| > 0$. Define $p = \Exp_{\theta \sim \calP}|\mu_{\theta}|$, $ q = |\nu|$ and let $V^* = \sup_{\fraka}\Exp_{\theta \sim \calP}\mu_{\theta}[\{L(\fraka(X),\theta) = 0\}]$, $V_0 = \sup_{a \in \calA}\Pr_{\theta \sim \calP}[L(a,\theta) = 0]$. Then,
	\begin{eqnarray}\label{Gen_Fano_Eq}
	\text{either} \quad V^* \le p V_0 & \text{or} & \Exp_{\theta \sim \calP} D_f(\mu_{\theta},\nu) \ge \phi_f(V_*,q\cdot V_0;p,q).
	\end{eqnarray}
	\end{thm}
	\begin{proof}We follow along the lines of the proofs of Lemma 3 and Theorem 2~\cite{chen2016bayes}, but first we introduce some notation. Let $\calP \otimes \nu$ denote the product measure between $\calP$ and $\nu$, and let $\calP * \{\mu_{\theta}\}$ denote the coupled measure with density $\rm(\calP * \{\mu_{\theta})\}(\theta,X) = \rmd\calP(\theta)\cdot \rm\mu_{\theta}(X)$. Also, given a measure $\eta$ on $(\Theta,\calG) \times (\calX, \calF)$, define the \emph{value} of $\fraka$ as the mass of all pairs $(\theta,X)$ which incur zero loss
	\begin{eqnarray}
	V_{\fraka}(\eta) = \eta(\{(\theta,X): \I(L(\fraka(X),\theta) = 0\}).
	\end{eqnarray}
	Defining the event $A := \{(\theta,X): \I(L(\fraka(X),\theta) = 0\}$, we have
	\begin{eqnarray*}
	\Exp_{\theta \sim \calP} D_f(\mu_{\theta},\nu) &=& \int f\left(\frac{\rmd\mu_{\theta}}{\rmd\nu}\right)d(\calP \otimes \nu)\\
	&=& \int f\left(\frac{\rmd\calP \cdot \rmd\mu_{\theta}}{d\calP \cdot \rmd\nu}\right)\rmd(\calP \otimes \nu)\\
	&=& D_f\left(\calP * \{\mu_{\theta}\},(\calP \otimes \nu)\right)\\
	&\overset{(i)}{\ge}& \phi_f\left((\calP * \{\mu_{\theta}\})(A),\calP \otimes \nu(A);|\calP * \{\mu_{\theta}\}|,|\calP \otimes \nu|\right)\\
	&\overset{(ii)}{=}& \phi_f\left(V_{\fraka}(\calP * \{\mu_{\theta}\}),V_{\fraka}(\calP \otimes \nu);p,q\right),
	\end{eqnarray*}
	where $(i)$ follows from the Data Processing inequality in Corollary~\ref{phi_cor} Part 1, and for $(ii)$ used the definition of $V^{\fraka}$ and the definitions $p = |\calP * \{\mu_{\theta}\}|$ and $q = |\calP \otimes \nu| = |\calP||\nu| = |\nu|$. 

	To wrap up, suppose that $V^* > p V_0$. We first note that $V_{\fraka}(\calP \otimes \nu) \le  |(\calP \otimes \nu)|\cdot V_0= q V_0$, since $X$ and $\theta$ are independent under $\calP \otimes \nu$. Moreover, for any $\epsilon > 0$, there exists a decision rule $\fraka$ for which
	\begin{eqnarray}
	p = |\calP * \{\mu_{\theta}\}| \ge V_{\fraka}(\calP * \{\mu_{\theta}\}) > V^* - \epsilon.
	\end{eqnarray}
	Taking $\epsilon$ small enough $V^* - \epsilon > pV_0$, we have that 
	\begin{eqnarray}
	V_{\fraka}(\calP * \{\mu_{\theta}\}) > V^* - \epsilon > pV_0 = \frac{p}{q}(qV_0) \ge \frac{p}{q}V_{\fraka}(\{\calP \otimes \nu\}).
	\end{eqnarray}
	By Part 3 Corollary~\ref{phi_cor}, applied first to the '$b$' argument and then to the `$a$' argument, we have 
	\begin{eqnarray*}
	\phi_f(V_{\fraka}(\calP * \{\mu_{\theta}\});V_{\fraka}(\{\calP \otimes \nu\});p,q) &\ge& \phi_f(V_{\fraka}(\calP * \{\mu_{\theta}\},qV_0;p,q)\\
	 &\ge& \phi_f(V^* - \epsilon,qV_0;p,q).
	\end{eqnarray*} 
	Since $\phi_f(a,b;p,q)$ is convex (Corollary~\ref{phi_cor}, Part 2), and therefore continuous, in its `$a$' argument for $a \in [0,p]$, and since $V^* \le p$, taking $\epsilon \to 0$ concludes. 
	\end{proof}

\subsection{Application to Lower Bounds for Truncated Distributions}
	The flexibility to work with non-normalized measures allows us to prove lower bounds on the probability of taking a zero-loss action, restricted to some small parts of the space. We will also specialize Theorem~\ref{Gen_Bayes_risk} to the case of the $\chi^2+1$ divergence defined in Definition~\ref{chi_plus_one_definition}.

	\begin{prop}\label{Application_Fano_rop} Let $\{\Prit_{\theta}\}_{\theta \in \Theta}$ denote a family of probability measures on a space $(\calX,\calF)$ indexed by $\theta \in \Theta$. Let $\calP$ denote a probability distribution on $(\Theta,\calG)$, and let $\{A_{\theta}\} \in \calF$ such that $\{(\theta,A_{\theta})\}$ is $(\calG,\calF)$ measurable. Define $\Prit_{\theta}[\cdot;A_{\theta}]$ denote the subprobability measure given by 
	\begin{eqnarray}
	\Prit_{\theta}[B;A_{\theta}] = \Prit_{\theta}[B \cap A_{\theta}].
	\end{eqnarray}
	Given a loss function $\calL:\calA \times \Theta \to \{0,1\}$, define 
	\begin{multline}
	V^* = \sup_{\fraka}\Exp_{\theta \sim \calP}\Prit_{\theta}\left[\{L(\fraka(X),\theta) = 0\} \cap A_{\theta}\right], \quad \text{and} \quad V_0 = \sup_{a \in \calA}\Pr_{\theta \sim \calP}\left[\{L(a,\theta) = 0\}\right].
	\end{multline}
	Letting $p = \Exp_{\theta \sim \calP}\Prit_{\theta}[A_{\theta}]$, then either $V^* \ge pV_0 $ or, for any probability measure $\Qit$ on $(\calX,\calF)$,
	\begin{eqnarray}\label{GenFanoConsequenceEq}
	\Exp_{\theta \sim\calP}D_{f}[\Prit_{\theta}[\cdot;A_{\theta}],\Qit] &\ge&  \phi_{f}(V^*,V_0;p,1).
	\end{eqnarray}
	In particular, if we chose $f(x) = x^2$, then we have
	\begin{eqnarray}
	V^* \le V_0 + \sqrt{V_0(1-V_0)\Exp_{\theta \sim\calP}D_{\chi^2 + 1}[\Prit_{\theta}[\cdot;A_{\theta}],\Qit]}. 
	\end{eqnarray}
	\end{prop}

	\begin{proof}[Proof of Proposition~\ref{chi_sq_fano_prop}]
	First we establish Equation~\eqref{GenFanoConsequenceEq}, by applying Theorem~\ref{GenFano} with $\mu_{\theta} = \Pr_{\theta}[\cdot;A_{\theta}]$ and $\nu = \Qit$. Since $\Qit$ is a probability measure, $q = 1$. This gives
	\begin{eqnarray*}
	\Exp_{\theta \sim\calP}D_{f}[\Pr_{\theta}[\cdot;A_{\theta}],\Qit]  \ge \phi_{f}(V^*,qV_0;p,q) = \phi_f(V^*,V_0;p,1). 
	\end{eqnarray*}
	In the case that $f(x) = x^2$, we compute
	\begin{eqnarray*}
	\phi_{\chi^2+1}(V^*,V_0;p,1) &=& \frac{(V^*)^2}{V_0} + \frac{(p - V^*)^2}{(1-V_0)}\\
	&=& \frac{(V^*)^2(1-V_0) + V_0(p-V^*)^2}{V_0(1-V_0)}\\
	&=& \frac{(V^*)^2 - 2pV_0V^* + p^2 V_0}{V_0(1-V_0)}\\
	&\overset{(i)}{\ge}& \frac{(V^*)^2 - 2pV_0V^* + p^2 V_0^2}{V_0(1-V_0)}\\
	&=& \frac{(V^* - pV_0)^2 }{V_0(1-V_0)}.
	\end{eqnarray*}
	where $(i)$ uses that $V_0 \le 1$. Thus, either $V^* \le pV_0 \le V_0$, or otherwise, 
	\begin{eqnarray}
	V^* \le pV_0 + \sqrt{V_0(1-V_0)\Exp_{\theta \sim\calP}D_{f}[\Prit_{\theta}[\cdot;A_{\theta}],\Qit]} \\
	\le V_0 + \sqrt{V_0(1-V_0)\Exp_{\theta \sim\calP}D_{f}[\Prit_{\theta}[\cdot;A_{\theta}],\Qit]}~.
	\end{eqnarray}
	\end{proof}
	\begin{rem}\label{TruncationRemark2} Examining the result of the above proposition, we see why it is so useful to consider the truncated measures $\Pr_{\theta}[\cdot;A_{\theta}]$ over their conditional analogues $\Pr_{\theta}[\cdot\big{|}A_{\theta}]$. This is because, if we instead proved a lower bound in terms of the latter, we would have to keep track of a normalization constants $\Pr_{\theta}[A_{\theta}]$, which would vary for all $\theta \in \Theta$. In particular, our case of interest considers the events $A_{\theta}^k = \{d\langle \vi, \theta\rangle^2 \le \tau_i,~ \forall i \in [k]\}$ from Section~\ref{sec_chi_est}. Because $\vi$ are chosen adaptively, it is quite difficult to control the normalization constant $\Pr_{\theta}[A_k^{\theta}]$ over all $\theta \in \calS^{d-1}$.
	\end{rem}

\subsection{Proof of Proposition~\ref{truncChi_detec}}
	For ease of notation, let $\overline{\Prit}_A =  \overline{\Prit}[\cdot;\{A_{\theta}\}]$. Note that $\overline{\Prit} - \overline{\Prit}_A  \ge 0$, which impies that
	\begin{eqnarray}
	\int |\rmd\overline{\Prit} - \rmd\overline{\Prit}_A| = \int \rmd\overline{\Prit} - \rmd\overline{\Prit}_A = 1 - \overline{\Prit}_A(\calX) := 1-p,
	\end{eqnarray}
	so by the triangle inequality
	\begin{eqnarray*}
	&& \|\Qit - \overline{\Prit}\|_{TV} =\frac{1}{2}\int |\rmd\Qit(x) - \rmd\overline{\Prit}(x)|\\
	&\le& \frac{1}{2}\int |\rmd\Qit(x) - \rmd\Prit_A(x)| + \frac{1}{2} \int |\rmd\overline{\Prit} - \rmd\overline{\Prit}_A| = \frac{1}{2}\int |\rmd \Qit(x) - \rmd\Prit_A(x)| + \frac{1-p}{2},
	\end{eqnarray*}
	Next, since $\Qit$ is a probability measure,
	\begin{eqnarray*}
	\int |\rmd\Qit(x) - \rmd\Prit_A(x)| &=& \Exp_{\Qit} |\frac{\rmd\overline{\Prit}_A}{\rmd\Qit} - 1| \le \sqrt{\Exp_{\Qit} |\frac{\rmd\overline{\Prit}_A}{\rmd\Qit} - 1|^2}\\
	&=& \sqrt{\Exp_{\Qit} |\frac{\rmd\overline{\Prit}_A}{\rmd\Qit}|^2  + 1 - 2\overline{\Prit}_{A}(\calX)} = \sqrt{\Exp_{\Qit} |\frac{d\overline{\Pr}_A}{\rmd\Q}|^2  + 1 - 2p}\\
	&=& \sqrt{\Exp_{\Qit} |\frac{\rmd\overline{\Prit}_A}{\rmd\Qit}|^2  - 1 + 2(1-p) } \le  \sqrt{\Exp_{\Qit} |\frac{\rmd\overline{\Prit}_A}{\rmd\Qit}|^2  - 1} + \sqrt{2(1-p)}.
	\end{eqnarray*}
	Putting pieces together yields the proof.

\subsection{Proof of Lemma~\ref{fDivProperties}}
The set $\{(\mu,\nu): \mu \ll \nu\}$ is convex, since if $\alpha \nu_1(A) + (1-\alpha)\nu_2(A) = 0$, then $\nu_1(A) = \nu_2(A) = 0$, and thus if $\mu_1 \ll \nu_1$ nad $\mu_2 \ll \nu_2$, then $\alpha \mu_1(A) + (1-\alpha)\mu_2(A) = 0$. Moreover, the perspective map $(x,y) \to yf(x/y)$ is jointly for convex $f$~\cite{boyd2004convex}, so that $\int f(\frac{\rmd\mu}{\rmd\nu})\rmd\nu$ is jointly convex in each argument. 

	For the second point , we see that that, by Jensen's inequality:
	\begin{eqnarray*}
	\int f(\frac{\rmd\mu}{\rmd\nu})\rmd\nu &=& |\nu|\int f(\frac{\rmd\mu}{\rmd\nu})\frac{\rmd\nu}{|\nu|} \ge |\nu|f( \int \frac{\rmd\mu}{\rmd\nu}\frac{\rmd\nu}{|\nu|} ) \\
	&=& |\nu|f( \frac{1}{|\nu|}\int \rmd\mu ) =  |\nu|f(\frac{|\mu|}{|\nu|}), 
	\end{eqnarray*}
	so the result holds as long as $f'(\infty) \ge 0$.

	Third, let $g(t) = f(t;p,q) = |\nu|f(t\frac{|\mu|}{|\nu|})$.Then $g'(\infty) = f'(\infty) \cdot |\mu|/|\nu| \cdot |\nu| = |\mu|f'(\infty)$. Thus,
	\begin{eqnarray*}
	D_{f}(\mu,\nu) &=& \int f(\frac{\rmd\mu}{\rmd\nu})d\nu + \mu(\{\rmd\nu = 0\})f'(\infty) \\
	&=& \int |\nu|f(\frac{|\mu|}{|\nu|}\cdot\frac{\rmd(\mu/|\mu|)}{\rmd(\nu/|\nu|)})\cdot \rmd\nu/|\nu| + (\frac{\mu}{|\mu|})(\{d\nu = 0\})\cdot |\mu|f'(\infty)\\
	&=& \int g(\frac{\rmd(\mu/|\mu|)}{\rmd(\nu/|\nu|)})\rmd\nu/|\nu| + \frac{\mu}{|\mu|}(\{\rmd\nu = 0\})g'(\infty)\\
	&=& D_{g}(\mu/|\mu|,\nu/|\nu|),
	\end{eqnarray*}
	as needed. For the fourth point point, note that for any constant $\alpha, (f+\alpha)'(\infty) = f(\infty)$. Thus,
	\begin{eqnarray*}
	D_{f+\alpha}(\mu,\nu) &=& \int \{f(\frac{\rmd\mu}{\rmd\nu})+\alpha\}d\nu + \mu(\{\rmd\nu = 0\})(f+\alpha)'(\infty)\\
	&=& \alpha |\nu| + \int f(\frac{\rmd\mu}{\rmd\nu})\rmd\nu + \mu(\{\rmd\nu = 0\})(f)'(\infty) = \alpha |\nu| + D_{f}(\mu,\nu).
	\end{eqnarray*}
	Similarly, since $(\beta f)'(\infty) = \beta f'(\infty)$, one has
	\begin{eqnarray*}
	D_{\beta f}(\mu,\nu) &=& \int \{\beta f(\frac{\rmd\mu}{\rmd\nu})+\alpha\}\rmd\nu + \mu(\{\rmd\nu = 0\})(\beta f)'(\infty)\\
	&=& \beta \int f(\frac{\rmd\mu}{\rmd\nu})\rmd\nu + \beta \mu(\{\rmd\nu = 0\})(f)'(\infty) = \beta D_{f}(\mu,\nu).
	\end{eqnarray*}
	Finally, the fifth point follows from the standard data-processing inequality \cite{chen2016bayes} in the case when $f$ is convex and $f(1) = 0$ and $\mu,\nu$ are both probability distributions. By the previous bound, the inequality can be extended to convex $f$ where $f(1)$ is not necessarily zero, and normalized $\mu,\nu$, by noting that
	\begin{eqnarray*}
	D_{f}(\mu,\nu) &=& D_{f - f(1)}(\mu,\nu) + f(1)|\nu|\\
	&\ge& D_{f - f(1)}(\mu\Gamma^{-1},\nu\Gamma^{-1}) + f(1)|\nu| \quad \text{ (classical data processing, e.g. Theorem 3.1 in Liese~\cite{liese2012phi})} \\
	&=& D_{f - f(1)}(\mu\Gamma^{-1},\nu\Gamma^{-1}) + f(1)|\nu\Gamma^{-1}| \quad \text{ ($\Gamma$ preserves total mass)} \\
	&=& D_{f}(\mu\Gamma^{-1},\nu\Gamma^{-1}). 
	\end{eqnarray*}
	To generalize to arbitrary finite, positive measures, we note that the function $f(t;|\mu|,\nu)$ is convex, so 
	\begin{eqnarray*}
	D_{f}(\mu,\nu) &=& D_{f(;|\mu|,\nu)}(\mu/|\mu|,\nu/|\nu|)\\
	&=& D_{f(;|\mu|,\nu)}(\frac{\mu}{|\mu|}\Gamma^{-1},\frac{\nu}{|\nu|}\Gamma^{-1})\\
	&=& D_{f}(\mu\Gamma^{-1},\nu\Gamma^{-1}).
	\end{eqnarray*}


\section{Proof of Lemma~\ref{SphereConcentation}\label{Concentration}}
We begin by invoking a result from spherical isoperimetry:
\begin{thm}[Spherical Isoperimetry,  page 211  in~\cite{boucheron2013concentration}]\label{SpherIso}
Let $f$ be an $1$-Lipschitz function on the sphere $\calS^{d-1}$. Then,
\begin{align}
\Pr_{\theta \sim \calS^{d-1}}\left[ f(\theta) \ge \mathrm{Median}(f) + t\right] \le \sup_{A \subset \calS^{d-1}}\Pr_{\theta \sim \calS^{d-1}}\left[\theta \in A_t\right] \le e^{-dt^2/2}.
\end{align}
where $A_t := \{\theta: \exists~ \theta' \in A \text{ such that } \|\theta - \theta'\|_2 \le t\}$.
\end{thm}
We can now prove Lemma~\ref{SphereConcentation}. Observe that $\mathrm{Median}(\langle \theta, v \rangle|) \le \sqrt{2/d}$, since by Markov's inequality,
\begin{eqnarray}
\Pr[|\langle \theta, v \rangle| \ge \sqrt{2/d}] \le \Exp[|\langle \theta, v \rangle|^2]/(2/d) = d/2 \cdot (1/d) = 1/2.
\end{eqnarray}
Since $\theta \mapsto |\langle \theta, v \rangle|$ is $1$-Lipschitz, Theorem~\ref{SpherIso} yields
\begin{eqnarray}
\Pr[\langle \theta, v \rangle| \ge \sqrt{2/d} + t] \le e^{-dt^2/2}.
\end{eqnarray}
Replacing $t$ with $\sqrt{t/d}$ concludes the proof.

\end{document}